%% file: main.tex
\documentclass{article}
    \PassOptionsToPackage{numbers, compress}{natbib}
    \usepackage[preprint]{sty}
\usepackage[utf8]{inputenc} 
\usepackage[T1]{fontenc}    
\usepackage{hyperref}       
\usepackage{url}            
\usepackage{booktabs}       
\usepackage{amsfonts}       
\usepackage{nicefrac}       
\usepackage{microtype}      
\usepackage{xcolor}         
\usepackage{wrapfig}
\usepackage{graphicx}
\usepackage{float}
\usepackage{amsmath}
\usepackage{amssymb}
\usepackage{mathtools}
\usepackage{amsthm}
\usepackage[algo2e,ruled,noend,vlined,linesnumbered]{algorithm2e}
\usepackage{algorithm}
\usepackage{algorithmic}
\usepackage[capitalize,noabbrev]{cleveref}
\usepackage{nccmath}
\usepackage{enumitem}
\usepackage{graphicx}
\usepackage{placeins}
\usepackage{xspace}
\usepackage{multirow}
\title{HyperTime: Hyperparameter Optimization for Combating Temporal
Distribution Shifts}
\author{Shaokun Zhang$^{1}$,~ Yiran Wu$^{1}$, Zhonghua Zheng$^{2}$, Qingyun Wu$^{1}$, Chi Wang$^{3}$ \\
$^{1}$ Pennsylvania State University, State College, PA, USA  \\
$^{2}$ The University of Manchester, Manchester, UK  \\
$^{3}$ Microsoft Research, Redmond, Washington, USA\\
\small{\texttt{\{shaokun.zhang, yiran.wu@psu.edu, qingyun.wu\}@psu.edu}}, \\
\small{\texttt{zhonghua.zheng@manchester.ac.uk}}, \\
\small{\texttt{wang.chi@microsoft.com}}
}

\input{macro-math.tex}
\begin{document}
\maketitle

\begin{abstract}
In this work, we propose a hyperparameter optimization method named \emph{HyperTime} to find hyperparameters robust to potential temporal distribution shifts in the unseen test data. 
Our work is motivated by an important observation that it is, in many cases, possible to achieve temporally robust predictive performance via hyperparameter optimization. Based on this observation, we leverage the `worst-case-oriented' philosophy
from the robust optimization literature to help find such
robust hyperparameter configurations. 
HyperTime imposes a lexicographic priority order on average validation loss and worst-case validation loss over chronological validation sets.
We perform a theoretical analysis on the upper bound of the expected test loss, which reveals the unique advantages of our approach. We also demonstrate the strong empirical performance of the proposed method on multiple machine learning tasks with temporal distribution shifts.  
\end{abstract}

\section{Introduction}

\begin{wrapfigure}{r}{0.48\textwidth}
\label{fig:demo}
\begin{minipage}{0.45\textwidth}
\vspace{-17pt}
\begin{figure}[H]
    \centering
    \includegraphics[width=0.78\columnwidth]{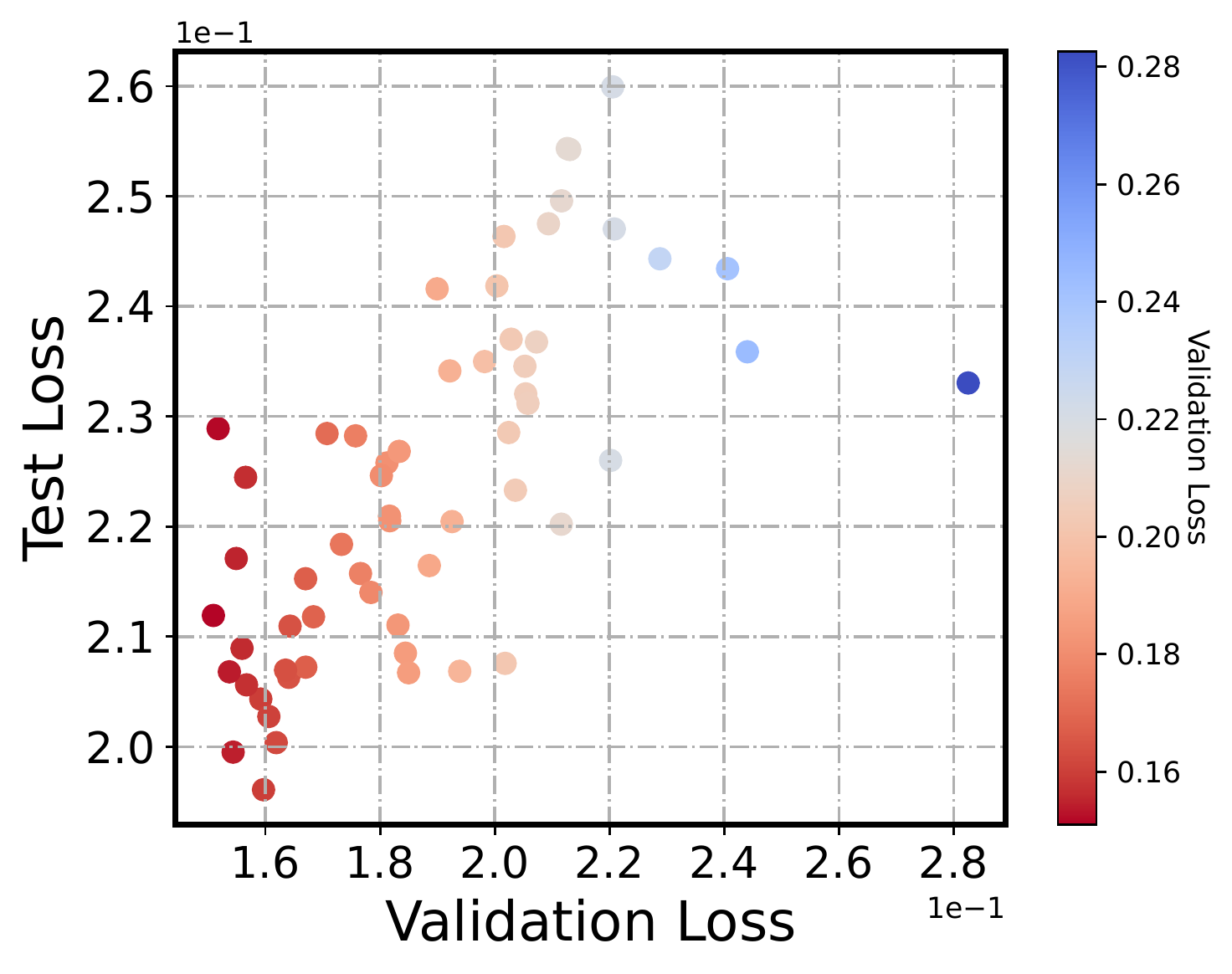}
    \caption{Validation loss vs. test loss on the Electricity dataset, where the validation and test data are from different time periods. Each point is a hyperparameter configuration randomly sampled from the search space. The loss here is (1- ROC\_AUC). }
    \label{fig:intro}
    \vspace{-5pt}
\end{figure}
\end{minipage}
\end{wrapfigure}

One major hurdle for machine learning systems to effectively perform over time is \emph{temporal distribution shifts}, which occur when the data distribution changes over time. 
If ignored, temporal distribution shifts may considerably degrade the predictive performance of the deployed machine learning models because of the data distribution mismatch during test time and train time~\cite{yao2022wildtime}. 
In recent years, many methods have been proposed to improve ML model's robustness to distribution shifts in general, including continual learning~\cite{adel2019continual,chaudhry2018riemannian}, invariant learning ~\cite{arjovsky2019invariant,yao2022improving}, self-supervised learning~\cite{chen2020simple,caron2020unsupervised}, and ensemble learning~\cite{izmailov2018averaging}. 
Although the methods mentioned above could potentially be adapted to handle temporal distribution shifts, the problem remains open and challenging: according to the evaluations from the Wild-Time benchmark~\cite{yao2022wildtime}, no existing invariant learning, continual learning, self-supervised learning, or ensemble learning approach is consistently more robust to temporal distribution shifts than vanilla empirical risk minimization (ERM). 

In this work, instead of intervening in the ERM-based model training procedure, we approach the problem from a different perspective, hyperparameter optimization (HPO).
It is known that some hyperparameters can affect the generalization capability~\cite{sun2021stagewise, bai2021ood} of ML models. It is unknown, however, whether we can achieve temporally robust predictive performance via HPO. Figure~\ref{fig:demo} presents a case study on the Electricity dataset with temporal shifts. We observe that: (a) models trained based on different hyperparameter configurations may exhibit vastly different performances on chronologically out-of-sample test data, and (b) validation loss is positively correlated with test loss in general, but when the validation loss is close to the lowest, configurations with the same validation loss may still have significantly different test losses. 
The first observation indicates that it is possible to build ML models that are more robust to distribution shifts by performing hyperparameter tuning and model selection. The second observation suggests that it can be challenging to find such robust hyperparameter configurations.

In this work, we apply a principle from distributionally robust optimization~\cite{delage2010distributionally,duchi2018learning, sinha2017certifying, bertsimas2018data} to the regime of hyperparameter optimization. 
More specifically, when doing HPO in environments with temporal distribution shifts, instead of optimizing the average predictive performance on validation data, which are typically sampled uniformly at random, we propose to (1) construct validation sets from different time periods and treat them as different proxies for the unseen test data, and (2) consider both the average validation loss and worse-case validation loss during HPO. 
Specifically, we use a multi-objective HPO approach which allows a lexicographic structure~\cite{fishburn1975axioms} on the objectives to reflect the different priorities of the concerned objectives. We treat the commonly used average validation loss as the primary objective and the worst-case performance among the different subsets of the validation data as the secondary objective. This gives us the opportunity to leverage the worst-case performance toward finding robust configurations while respecting the importance of average validation loss.  We provide theoretical analysis on the expected test loss of our method. The analysis shows the unique advantage of leveraging the average and worst-case validation loss in a lexicographic manner.

We verify the effectiveness of our method for tuning gradient-boosting trees and neural networks on a diverse range of datasets with temporal distribution shifts. Our method is also compatible with robust learning/training methods and is able to further boost their robustness to temporal distribution shifts. 
\section{Related Work}
A number of works are proposed to improve machine learning model's robustness when distribution shifts happen.  
One paradigm that can be applied is continual learning~\cite{adel2019continual,chaudhry2018riemannian,kirkpatrick2017overcoming,schwarz2018progress,zenke2017continual,gupta2020look,lopez2017gradient,rebuffi2017icarl,shin2017continual} algorithms. The target of continual learning is to learn from new data on the fly while not forgetting previously learned information. Another paradigm that can be applied is invariant learning\cite{ganin2016domain,long2015learning,sun2016deep,tzeng2014deep,xu2020adversarial,yue2019domain}. Invariant learning methods aim to learn invariant representation across different domains, which could also be adapted to distribution shifts. The representative works include CORAL~\cite{sun2016deep}, IRM~\cite{arjovsky2019invariant}, LISA~\cite{yao2022improving}, and GroupDRO~\cite{sagawa2019distributionally}.
Third, self-supervised learning~\cite{chen2020simple,caron2020unsupervised,ji2021power,shen2022connect} and ensemble learning methods~\cite{izmailov2018averaging,dong2020survey,webb2004multistrategy,oza2001online} are also applicable to mitigating distribution shifts. 
All the aforementioned existing work concerns the training procedure to improve the resulting model's robustness. They are mostly model-specific and not consistently more robust to temporal distribution shifts than vanilla ERM according to the Wild-Time benchmark~\cite{yao2022wildtime}.

To the best of our knowledge, no existing HPO method concerns the temporal distribution shift problem. The only relevant work is a robust neural network search method named \emph{NAS-OoD}~\cite{bai2021ood}, which searches for neural networks that generalize to out-of-distribution data under the differentiable neural architecture search paradigm~\cite{liu2018darts}. However, this method is not model-agnostic and is not directly applicable to mitigate temporal distribution shifts.

\section{Method}
In this section, we present the proposed HPO method for combating temporal distribution shifts.

\subsection{Notions, Notations, and Background}
Before introducing details of the proposed method, we first introduce notions and notations to be used throughout the paper and some background knowledge on hyperparameter optimization and temporal distribution shifts.

\begin{itemize}[leftmargin=*]
    \setlength\itemsep{-0.2em}
    \item $(\bx,y)$ denotes a specific supervised data instance where $\bx$ represents the feature and $y$ represents the label. $\cD = \{..., (\bx, y), ...\}$ denotes a supervised dataset in general. When necessary, we use \(\cD_{t_1:t_2}\) to denote the subset of the dataset within a certain time period, e.g., \(t_1\) to \(t_2\).  
    \item $c$ denotes a hyperparameter configuration in a particular hyperparameter search space $\cC$.
    \item $f$ denotes a machine learning model in general. When further details on the training data and hyperparameters are needed, we use sub-script $c$ and $\cD$ in $f_{c, \cD}$ to reflect that the model is constructed with a hyperparameter configuration $c$ and trained on dataset $\cD$. We use $f(\bx)$ to denote the inference process on $\bx$ outputting a predicted label. 
    \item $\Loss(f, \cD)$ denotes the predictive loss of an ML model $f$ on dataset $\cD$ under a particular loss function. For example, when Mean Squared Error is the loss function $\Loss(f, \cD) = \frac{1}{|\cD|}\sum_{(\bx, y) \in \cD} (f(\bx)  - y)^2$. 
    \item We use \([K]\) as a shorthand for the set of integer from \(1\) to \(K\), i.e., \( [K] \coloneqq \{1,2,...,K\}\).
\end{itemize}

In a supervised machine learning setting, given a training dataset $\cD_{\train}$, the ultimate goal is to build a model $f$ based on $\cD_{\train}$ that has the best expected predictive performance on some unseen test data. Since the test data are unseen, a validation dataset is typically reserved (e.g., by sampling a particular portion uniformly at random) from the available training data as a proxy to evaluate the predictive performance of the model on the unseen test data. In ML practice, validation loss is used ubiquitously as the primary metric for model selection in HPO and, more generally, AutoML. Specifically, a typical formulation of HPO is the following black-box optimization problem,
\begin{align} \label{eq:vanilla_hpo}
  \min_{c \in \cC } \Loss(f_{c, \cD_{\train}}, \cD_{\val}),
\end{align}
in which $\Loss(f_{c, \cD_{\train}}, \cD_{\val})$ is the valuation loss on $\cD_{\val}$ corresponding to hyperparameter configuration $c$, and the objective of an HPO method under this formulation is to effectively find a hyperparameter configuration with the best validation loss. This optimization process is a principled approach for building an ML model with good expected predictive performance on unseen test data when there is no distribution drift in the data (the expected predictive performance on the test data and validation data are supposed to be close according to theories in statistical machine learning~\cite{learning_from_data_book}). However, when there is indeed data distribution drift, the optimization objective specified in Eq.~\eqref{eq:vanilla_hpo} becomes questionable because of the mismatch between the predictive performance on the validation and test data due to distribution shifts.

\subsection{Robust HPO by Imposing Lexicographic Objectives} 
Our overarching insight for doing robust HPO is to construct a set of possible realizations of the unseen test data and take the worst-case realizations into consideration in the hyperparameter optimization objectives. 

To implement this idea, we first construct $K$ validation sets, denoted by $\{ \cD_{1}, \cD_{2}, ..., \cD_{K}\}$, which are possible realizations of the unseen test data. Based on the $K$ validation sets, we could obtain a set of validation losses denoted by $\{L_1(c), L_2(c), ..., L_K(c)\} $ respectively.   We further denote the average loss and the worst loss among the $K$ losses as,
\begin{align}
L_{\avg}(c) \coloneqq \frac{\sum_{i=1}^K L_k(c)}{K} , L_{\worst}(c) \coloneqq \max \{ L_k(c) \}_{k \in [K]}.
\end{align}

If the data distribution in the unseen test set follows the same distribution as in the validation data, optimizing the average loss $L_{\avg}(c)$ is presumably a good practice, which is also the standard practice in classical HPO when cross-validation is used. However, in the scenarios where temporal shifts exist, this assumption is no longer true, and better practice is needed. Inspired by the ``worst-case-oriented" philosophy in robust optimization~\cite{delage2010distributionally, bertsimas2018data}, we propose to incorporate the validation loss on
the fold with the worst predictive performance, i.e., $L_{\worst}(c)$,
as an additional objective for HPO.

\textbf{Lexicographic Hyperparameter Optimization.}
It remains a question how one should incorporate the worst-case performance into consideration, especially regarding its relationship with average performance. 
In this work, we propose to include both average validation loss and worst-case validation loss during HPO and impose a lexicographic priority order on them. More specifically, we include the ordered list $\mathbf{L}(c) = [L_{\avg}(c), L_{\worst}(c)]$ as objectives with a lexicographic structure, in which $L_{\avg}(c)$ is the objective with higher priority and $L_{\worst}(c)$ as the one with lower priority.
By doing so, we could find a hyperparameter configuration with both a good average validation loss and a good worst-case validation loss over the validation folds.
Put more formally, we formulate the HPO process as:

\begin{align} \label{eq:lexi_hpo}
 \LexiMin_{c \in \cC } \mathbf{L}(c),
\end{align}
in which $\LexiMin$ is the optimization procedure over an ordered list of objectives $\mathbf{L}(c)$, following the Lexicographic relations defined in~\cite{zhang2023targeted}. 
We use $L^{(i)}$ to denote the $i$-th element of the list $\mathbf{L}(c)$ in general. In our optimization function, $L^{(1)}$ and $L^{(2)}$ represents $L_{\avg}$ and $L_{\worst}$, respectively. Given any configurations $c$,  $c'$, and $I = \lvert \mathbf{L}(c) \rvert$ (with $I > 1$) optimization objectives with a lexicographic priority order, the definition of lexicographic relation (between any \(c' \in \cC \) and \(c \in \cC \)) is:
\begin{align} 
\label{eq:lexico_relation_better}
 \mathbf{L}(c')   =_{l}  \mathbf{L}(c)  ~ & \Leftrightarrow \forall i \in  [I]:  L^{(i)}(c') = L^{(i)}(c),\\ 
\label{eq:lexico_relation_eq}
  \mathbf{L}(c') \prec_{l}   \mathbf{L}(c) ~ & \Leftrightarrow   \exists i \in [I]:  \nonumber
   L^{(i)}(c') < L^{(i)}(c)  \land (\forall i' < i, L^{(i')}(c') = L^{(i')}(c) ),    \\   \nonumber
  \mathbf{L}(c') \preceq_{l}  \mathbf{L}(c)  ~ & \Leftrightarrow  \mathbf{L}(c') \prec_{l}   \mathbf{L}(c)   \lor  \mathbf{L}(c')  =_{l}  \mathbf{L}(c).
\end{align}
 
 The optimal point under $\LexiMin$ is called the \emph{lexi-optimal} point, which is any one element in hyperparameter configuration set \(\cC^* = \{ c \in \cC_*^{(I)} | \forall c' \neq c, \mathbf{L}(c) \preceq_{l} \mathbf{L}(c') \}\). Here $\cC_*^{I}$ is defined in the following recursive way: $ \cC_*^{(0)} = \cC$ and for $i \in [I]$, 
\begin{align} 
 &   \cC_*^{(i)} \coloneqq 
\{c \in \cC_*^{(i-1)} | L^{(i)}(\bx) \leq L_{*}^{(i)}*(1 + \kappa^{(i)}) \}, \\   
 &   L_{*}^{(i)} \coloneqq \inf_{c \in \cC_*^{i-1}} L^{(i)}(c), \nonumber
\end{align}

where $\kappa^{(i)}$ is a non-negative number, representing the percentage of performance compromise of the $i$-th objective to find choices with better performance on the low-priority objectives. 

Compared with directly using the average validation loss $L_{\avg}$ as the single optimization objective, $\LexiMin$ is able to incorporate an auxiliary objective $L_{\worst}$ by adding it as the secondary objective in lexicographic preference. 
In this way, the optimization of $L_{\worst}$ only matters when the more important objective $L_{\avg}$ is well-optimized, i.e., within its optimality tolerance range.  
Compared to classical multi-objective HPO approaches, $\LexiMin$ is able to incorporate the intuition that the average loss shall be prioritized. We modify the HPO solution designed for this type of $\LexiMin$ problem originally proposed in~\cite{zhang2023targeted} to solve our problem after constructing the objectives. We include the algorithm details in the Appendix~\ref{appendix_lexiflow}. 

\paragraph{Remarks on validation data sets construction.} In addition to the lexicographic objectives on the average validation loss and the worst-case validation loss, we believe it is also important to consider how the validation shall be constructed.

The guiding principle for constructing the validation sets is that the validation sets should represent possible realizations of unseen data. Considering this principle and the potential temporal distribution shifts in the dataset, we propose to retain the chronological order over the data instances and sample the $K$ folds of validation data $\cD_1, \cD_2, ..., \cD_K$ at different time periods. 
More specifically, we first split the chronologically ordered training dataset into $K$ segments with $K-1$ time points $t_1,...,t_{K-1}$ in addition to the starting point \(t_0\) and the end point \(t_{K}\) (the actual value of the time points can be application dependent). We then ensure $ \forall k \in [K]$ the validation set $\cD_{k}$ is sampled from time period between $t_{k-1}$ to $t_k$.  By doing so we have a collection of diverse validation sets representative of the potentially shifted data distributions in the available training set. 
\begin{wrapfigure}{r}{0.5\textwidth}
\begin{minipage}{0.5\textwidth}
\vspace{-17pt}
\begin{figure}[H]
    \centering
      \includegraphics[width=0.9\columnwidth]{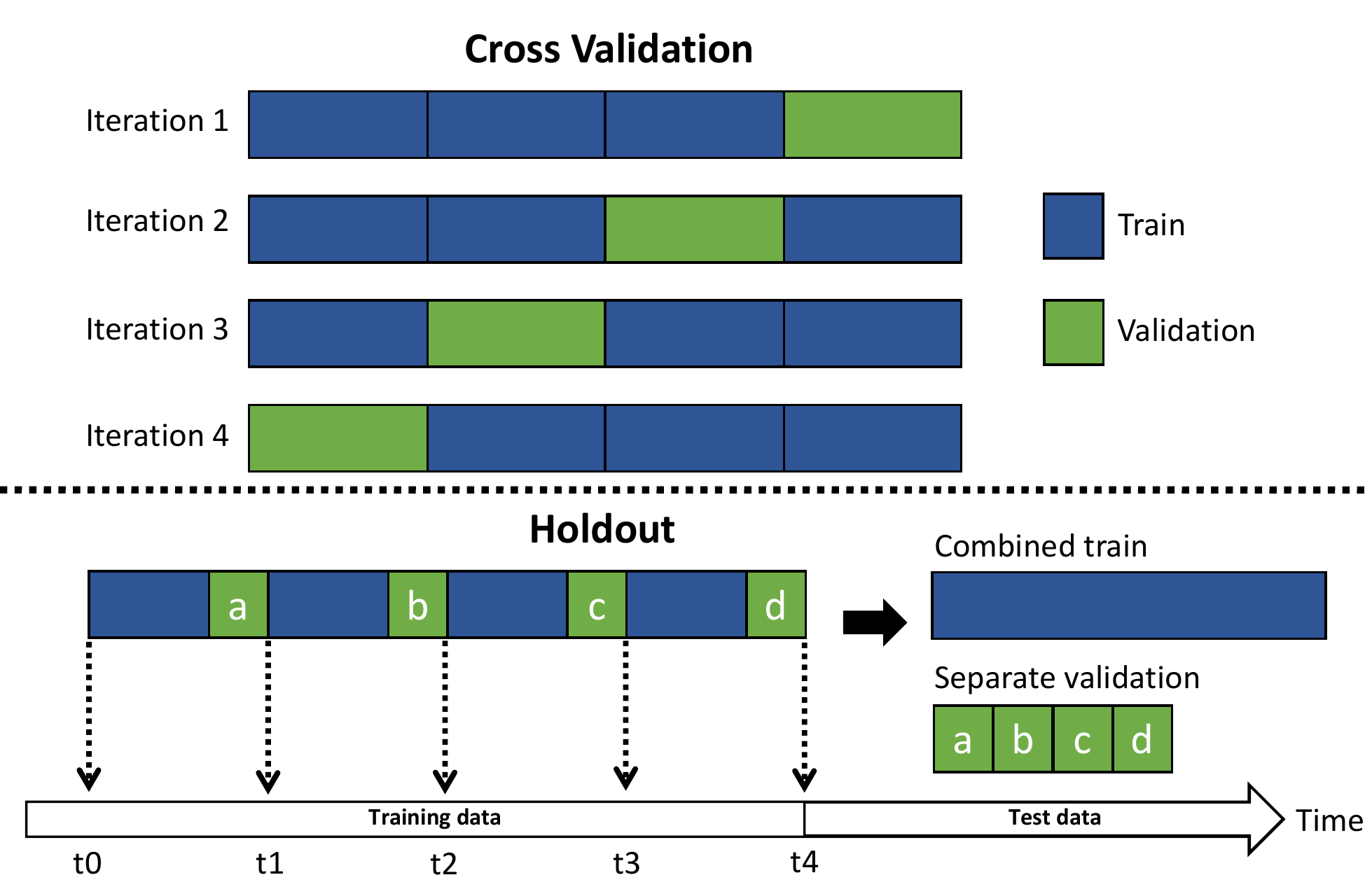}
    \vspace{-5pt}
    \caption{Chronological validation data sets construction with Cross Validation and Holdout strategies.}
    \label{fig:validation_sets_construction}
    \vspace{-5pt}
\end{figure}
\end{minipage}
\end{wrapfigure}
Depending on whether cross-validation or holdout is preferred, the validation set construction strategy and the corresponding calculation of validation losses in both typical cross-validation and holdout are visualized in Figure~\ref{fig:validation_sets_construction} and detailed formally as follows, in which we use \(\cD\) to denote the available dataset: \textbf{(1) Cross-validation:} Each evaluation of a particular configuration $c$ involves $K$ iterations of model training and evaluation. In the $i$-th iteration, the set $\cD_{i} = \cD_{t_{k-1}:t_k}$ is considered the validation set and the rest training set. And we have $L_i(c) \coloneqq \Loss(f_{c, (\cD \setminus \cD_{i})}, \cD_{i}) $ for $i \in [K]$. \textbf{(2) Holdout:} In this case, the evaluation of each configuration only involves training one single model with $K$ validation steps. The $k$-th validation set is $ \cD_{i} = \cD_{t_{k-1}':t_{k}} $ in which $t_{k-1} <t_{k-1}' < t_{k}$, and the data excluding the $K$-folds of valiation sets, i.e., $ \cD \setminus (\cD_1 + \cD_2 + ..., +  \cD_K) $, are used to train a model. And we have $L_i(c) \coloneqq  \Loss(f_{c, (\cD \setminus (\cD_{1} + \cD_{2} + ..., + \cD_{K}))}, \cD_{i}) $ for $i \in [K]$.

Although cross-validation is usually the preferred method because it allows models to train on multiple train-test splits, each evaluation of a particular configuration is more expensive than the holdout strategy (approximately $K$ times larger),
especially in hyperparameter tuning which depends on a large number of configuration evaluation processes. 
Therefore, we suggest choosing the validation sets construction method according to the detailed information of the scenario like data size, model types, resource limit, etc.

\section{Theoretical Analysis}
Following the same spirit as previous works on mitigating distribution shifts occurred with time in data stream~\cite{mallick2022matchmaker,hentschel2019online}, we assume that among the (K) validation sets from previous time periods, there exists one particular set that shares the most similar data distribution with the unseen test data at recent time periods. 
The optimal configuration is the one that performs best on this particular set. This assumption is an important relaxation of the full i.i.d. assumption required by existing HPO algorithms~\cite{bergstra2012random,duchi2018learning}.
We further introduce the following definitions to facilitate our analysis.

\begin{itemize}[leftmargin=*]
    \vspace{-4mm}
    \setlength\itemsep{-0.2em}
    \item Best configuration on the \(k\)-th validation data set \(c^*_{k}\): \(  c_{k}^* \coloneqq   \argmin\limits_{c \in \cC } \Loss(f_{c, \cD_{\train}}, \cD_{k}).\)
    \item Best average validation loss \(L^*_{\avg}\): 
    \(L^*_{\avg} \coloneqq \argmin\limits_{c \in \cC} L_{\avg}(c).\)
    \item We use \(k^*\) to denote the index of the validation set that shares the most similar data distribution with the unseen test data. In other words, validation set \( D_{{k}^*}\) shares the most similar data distribution with $D_{\test}$. 
    \item We use \(\hat c\) to denote the hyperparameter selected by our method.
\vspace{-4mm}
\end{itemize}

As defined above, the configuration \(c_{k^*}^*\) is the optimal configuration. However, \(k^*\) is unknown a prior without the test data. We provide the following bound on the validation loss of our selected configuration \(\hat c\) and a proof sketch as well as the detailed proof in Appendix~\ref{proof-theorem1} for Theorem~\ref{theorem_bound}.
 
\begin{thm} \label{theorem_bound}
When $\kappa \geq \frac{L_{\avg}(c^*_{k^*})}{L^*_{\avg}} - 1$, with probability at least \(1-\epsilon\) \((\epsilon \in (0,1))\), we have the following bounds on the expected test loss of the model with our selected configuration \(\hat c\),

\begin{align}
   \mathbb{E}[\Loss(f_{\hat c}, \cD_{\test})] \leq 
   \begin{cases}
      (1+\kappa)  L_{avg}(c_{k^*}^*) + \sqrt{\frac{\beta \ln (2/\epsilon) }{2|\cD_{\val}|}},   
      & \text{if~~}   L_{k^*}(\hat c) \leq L_{avg}(\hat c)    \\
      L_{\worst}(c_{k^*}^*) + \sqrt{\frac{\beta \ln (2/\epsilon) }{2|\cD_{\val}|}}, & \text{otherwise}   \nonumber
   \end{cases}
\end{align}

in which \(\beta\) is the upper bound on the loss. E.g., in binary classification task with 1-accuracy as the loss metric, \(\beta = 1\). 
\end{thm}
\begin{remark}[The role of \(\kappa\)]
According to the analysis in Appendix~\ref{proof-theorem1}, we have: \textbf{(1)} When \(L_{k^*}(\hat c) \leq L_{avg}(\hat c)\), a smaller \(\kappa\) shall be preferred. In fact, under this case, if we set \(\kappa\) to 0, and the method recovers the naive alternative, which uses the average validation loss as the HPO objective. \textbf{(2)} When \(L_{k^*}(\hat c) > L_{avg}(\hat c)\), using the average validation loss is no longer a good strategy as it may make the expected test loss \(\mathbb{E}[\Loss(f_{\hat c}, \cD_{\test})]\) as large as \(K L_{\worst}(c_{k^*}^*) + \sqrt{\frac{\beta \ln (2/\epsilon) }{2|\cD_{\val}|}}\). With our method, as long as \(\kappa\) satisfies \(\kappa \geq \frac{L_{\avg}(c^*_{k^*})}{L^*_{\avg}} - 1\),  \(\mathbb{E}[\Loss(f_{\hat c}, \cD_{\test})]\) is upper bounded by  \(L_{\worst}(c_{k^*}^*) + \sqrt{\frac{\beta \ln (2/\epsilon) }{2|\cD_{\val}|}}\) with high probability. 
Considering the fact that \(k^*\) is unknown a prior (in other words, which fold of the validation data is most similar with the test data is unknown a prior), both case (I) and case (II) may happen. Our method is able to properly bound the expected test loss in both cases despite the value of \(k^*\) is unknown. 
\end{remark}

\section{Experiments}
We first evaluate our method (\emph{HyperTime}) on the gradient-boosting trees and neural networks tuning tasks in Section~\ref{subsec:effectivenss} to verify the effectiveness of our method. 
We further perform in-depth investigations in Section~\ref{subsec:investigation} to (1) provide a better understanding of the important contributing factors in our method; and (2) study the compatibility of our method with robust training methods. If not otherwise specified, all the results in our evaluation are averaged over five different random seeds. 

\subsection{Effectiveness}
\label{subsec:effectivenss}
In this subsection, we show the off-the-shelf effectiveness of our proposed method for tuning tree-based boosting methods and deep neural networks.  We include three single objective HPO methods as baselines in all the evaluations, including randomized direct search method~\cite{wu2021frugal} (CFO), bayesian optimization HPO algorithm~\cite{bergstra2011algorithms} (BO), and multiple multi-fidelity HPO algorithm~\cite{li2017hyperband} (HB), which search for the best configuration that maximizes the average validation losses. 
In the task of boosting trees tuning, we also include the learners with default configuration, as baselines. This baseline can be considered as an ERM method under the tree-based boosting framework. 
In the task of deep neural network tuning, we include state-of-the-art robust training methods (including a vanilla ERM as well) for comparison. The detailed search spaces for each learner are included in Appendix~\ref{append:evaluation}. 

We use three metrics to perform evaluations on the test set, which could reveal the test performance of a method from multiple aspects. \textbf{(1) Average performance}: Average performance of all test folds. It reflects the overall performance of a specific method, and it is typically considered the most important metric in practice. \textbf{(2) Worst fold performance}: Worst fold performance across all test folds. It reflects the performance of a specific method in the worst cases. \textbf{(3) Winning fold number}: Number of test folds achieving the best performance compared with other methods. When temporal distribution shift happens, assuming each test fold follows one specific data distribution, winning fold number could reflect the number of cases in which a specific method works best compared with other methods.

\subsubsection{Tuning tree-based boosting methods}
\label{sec:tabular_experiments}

We first perform the evaluation for tuning different gradient-boosting trees on three tabular datasets, including two large-scale datasets Vessel Power Estimation~\cite{malinin2023shifts} and Urban Temperature Prediction, and a relatively small dataset Electricity~\cite{mallick2022matchmaker} to cover a wide use cases. The detailed information and the reasons for choosing these three datasets are shown in Appendix~\ref{appendix:dataset}. We tune XGBoost on the Electricity and Vessel Power Estimation datasets, and LightGBM on the Urban Temperature Prediction dataset~\cite{KAYETAL15}. Note that in CFO, we use the conventionally used validation data set construction, i.e., constructing validation sets by randomly sampling from shuffled datasets.
We report the average test loss, worst fold test loss with corresponding standard deviation, and the winning fold number in Table~\ref{tab:tabular}. 
Compared with all baselines, HyperTime achieves the best performance in terms of both average performance and the worst fold performance on all three datasets. It indicates our method could indeed help find hyperparameter configurations with relatively robust performance during test time.

\begin{table}[H]
\caption{Test time performance of HyperTime and baselines for tuning gradient-boosting trees on different datasets. We show the average test loss (Test-average), and average worst fold test loss (Test-worst) across test folds with 5 seeds respectively. The losses are the lower the better. For each method, we also show the number of folds achieving the best results compared with other methods, i.e., winning fold num (WN), which is the higher the better.} 
\label{tab:tabular}
\resizebox{\linewidth}{!}{
\begin{tabular}{cccccccccc}
\toprule[1.5pt]
\multicolumn{1}{c|}{}          & \multicolumn{3}{c|}{\textbf{Vessel Power}}                            & \multicolumn{3}{c|}{\textbf{Temperature}}                             & \multicolumn{3}{c}{\textbf{Electricity}}                           \\ \hline
\multicolumn{1}{c|}{Metric}    & Test-average         & Test-worst           & \multicolumn{1}{c|}{WN} & Test-average         & Test-worst           & \multicolumn{1}{c|}{WN} & Test-average         & Test-worst           & WN                   \\ \hline
\multicolumn{1}{c|}{Default}   & 1239.71              & 1936.27              & \multicolumn{1}{c|}{2}  & 1.1531               & 1.2831               & \multicolumn{1}{c|}{0}  & 0.1699               & 0.2186               & 2                    \\
\multicolumn{1}{c|}{CFO}       & 1475.69 (436.84)     & 2403.61 (538.37)     & \multicolumn{1}{c|}{0}  & 1.093 (0.049)        & 1.205 (0.058)        & \multicolumn{1}{c|}{0}  & 0.1781 (0.049)       & 0.2274 (0.048)       & 0                    \\
\multicolumn{1}{c|}{BO}        & 1966.80 (551.51)     & 2929.13 (594.74)     & \multicolumn{1}{c|}{0}  & 1.071 (0.038)        & 1.160 (0.038)        & \multicolumn{1}{c|}{0}  & 0.1744 (0.045)       & 0.2165 (0.044)       & 0                    \\
\multicolumn{1}{c|}{HB}        & 1757.17 (578.28)     & 2923.18 (758.25)     & \multicolumn{1}{c|}{0}  & 1.091 (0.048)        & 1.194 (0.044)        & \multicolumn{1}{c|}{0}  & 0.1782 (0.048)       & 0.2321 (0.043)       & 0                    \\
\multicolumn{1}{c|}{HyperTime} & \textbf{1108.97 (152.51)}     & \textbf{1397.14 (166.38)}     & \multicolumn{1}{c|}{\textbf{5}}  & \textbf{1.064 (0.037)}        & \textbf{1.149 (0.035)}        & \multicolumn{1}{c|}{\textbf{7}}  & \textbf{0.1653 (0.042)}       & \textbf{0.2112 (0.043)}       & \textbf{4}                    \\ \bottomrule[1.5pt]
\multicolumn{1}{l}{}           & \multicolumn{1}{l}{} & \multicolumn{1}{l}{} & \multicolumn{1}{l}{}    & \multicolumn{1}{l}{} & \multicolumn{1}{l}{} & \multicolumn{1}{l}{}    & \multicolumn{1}{l}{} & \multicolumn{1}{l}{} & \multicolumn{1}{l}{}  \\ 
\multicolumn{1}{l}{}           & \multicolumn{1}{l}{} & \multicolumn{1}{l}{} & \multicolumn{1}{l}{}    & \multicolumn{1}{l}{} & \multicolumn{1}{l}{} & \multicolumn{1}{l}{}    & \multicolumn{1}{l}{} & \multicolumn{1}{l}{} & \multicolumn{1}{l}{}
\end{tabular}}
\vskip -0.3in
\end{table}

We also present the predictive performance on each fold of the test data in Figure~\ref{fig:mainexp_tabular_folds}. Figure~\ref{fig:mainexp_tabular_folds} shows that HyperTime is consistently better than the baseline methods on different test folds in most cases. 
Although there are cases where the baseline methods have better performance than HyperTime on a specific fold, the margin of the differences is small. 
\begin{figure*}[htb!]
\vskip -0.1in
\begin{center}
\includegraphics[width=0.29\textwidth]{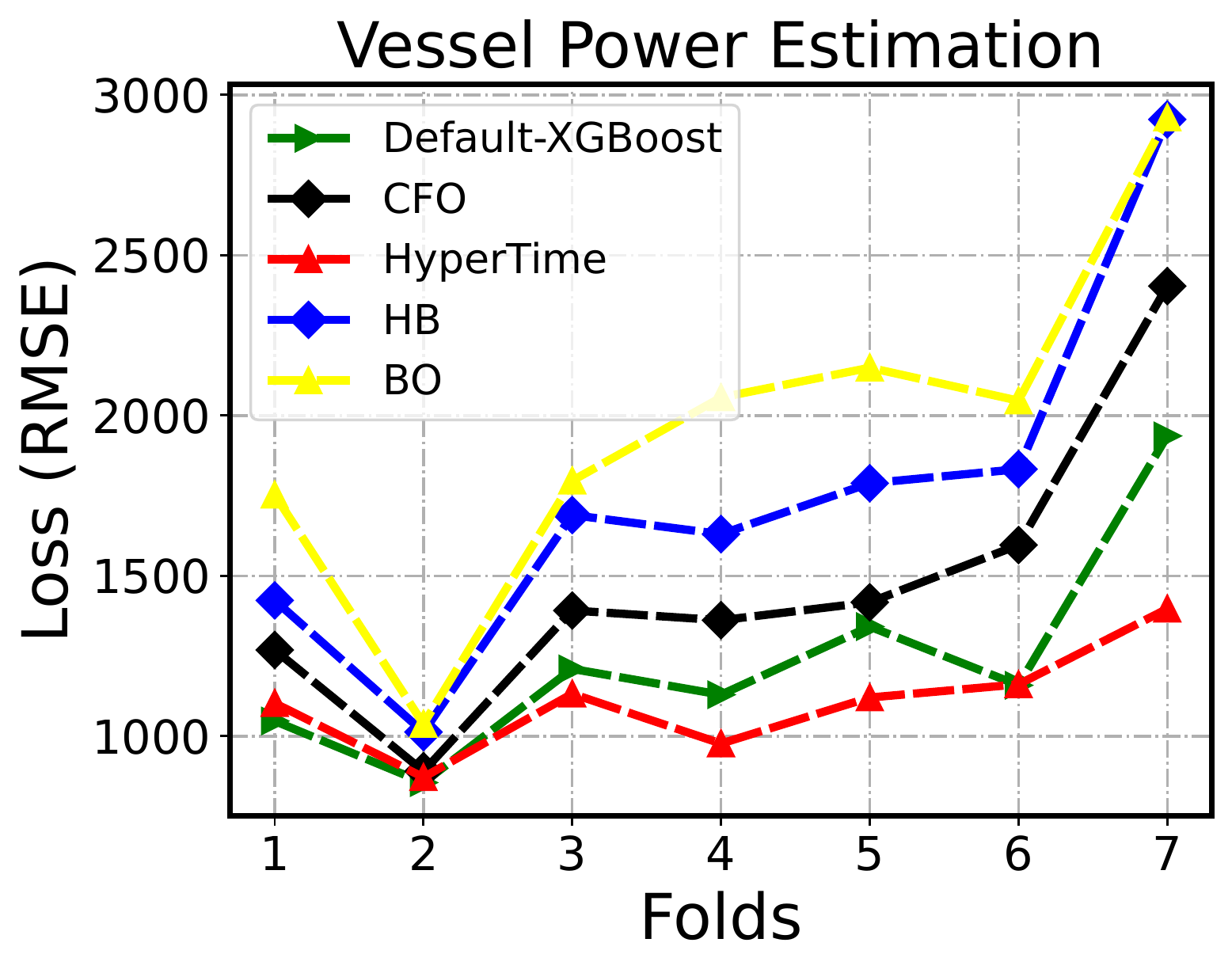}
\includegraphics[width=0.29\textwidth]{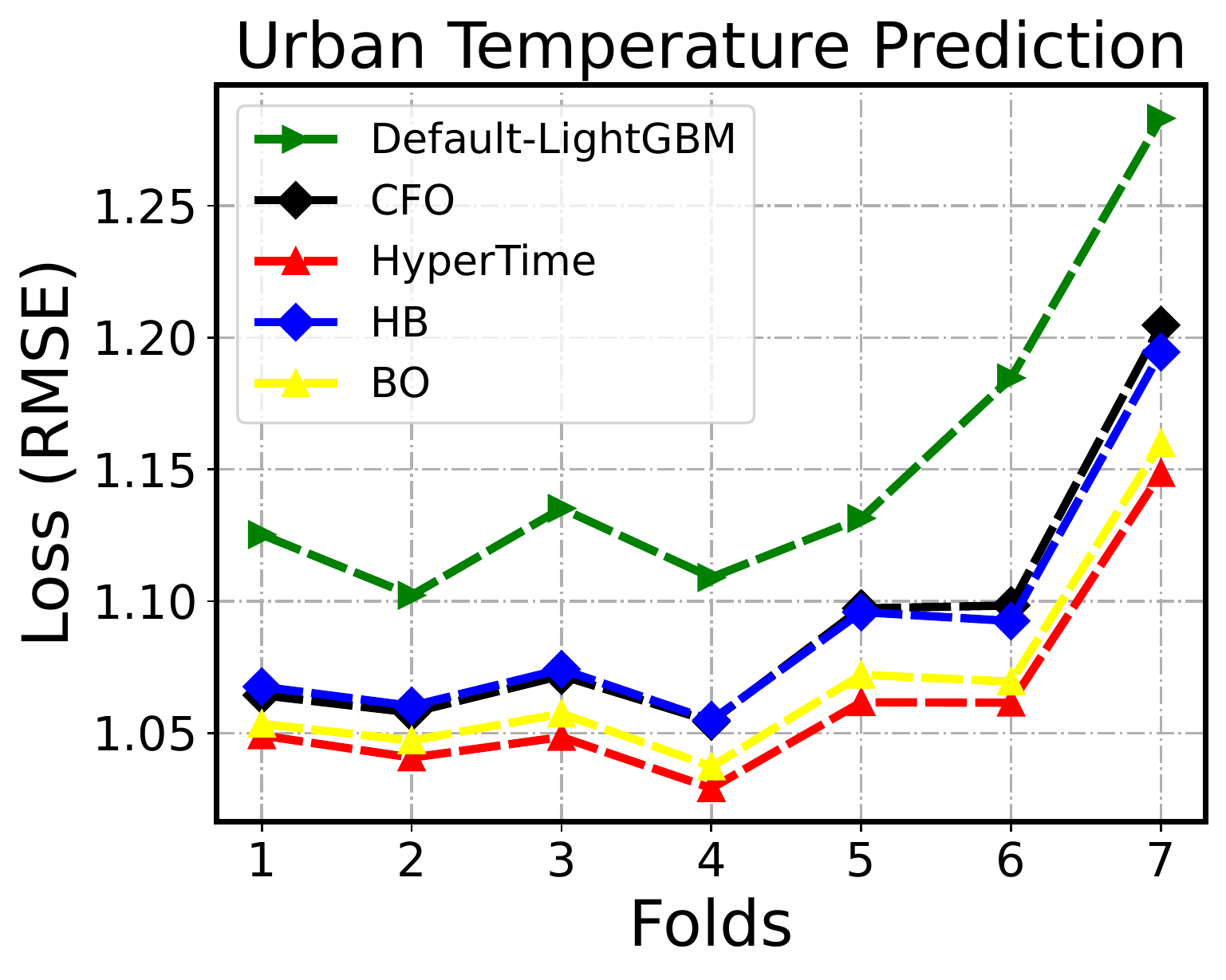}
\includegraphics[width=0.29\textwidth]{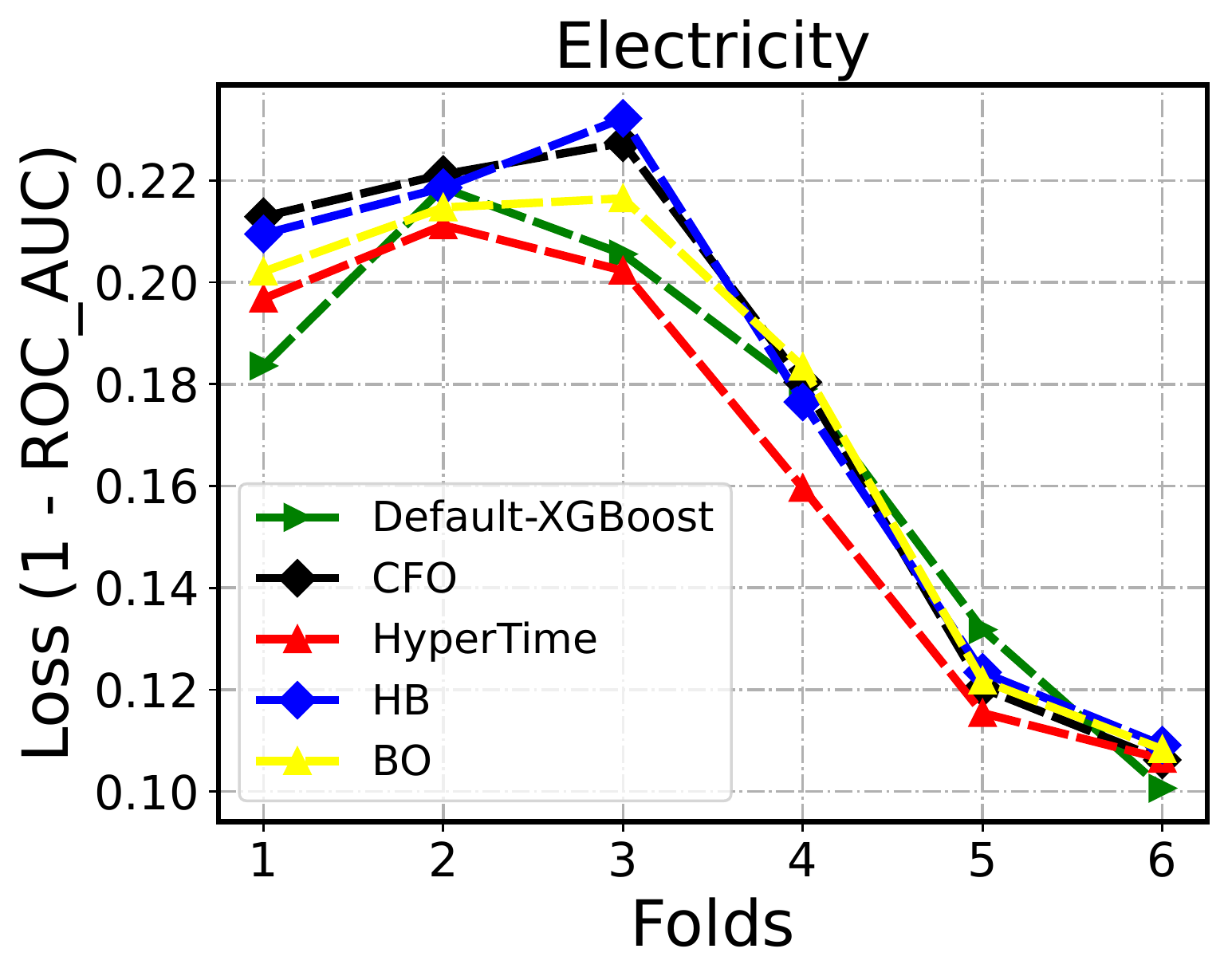}
\caption{Per fold test loss (lower the better) for tuning gradient-boosting trees on different datasets. The results are averaged over different random seeds. The results are from the same set of experiments with that in Table~\ref{tab:tabular}. \label{fig:mainexp_tabular_folds}
}
\end{center}
\vskip -0.15in
\end{figure*}

We also have an interesting observation: Vanilla HPO with the average validation loss as the objective is worse than the default learner in two of three datasets (2/3). This scenario also appears in pioneer works~\cite{bentejac2021comparative} and it reflects the motivation of our paper to some extent. Single-objective HPO algorithms only use the validation loss as the optimization objective, which may cause the searched architectures to overfit the validation data. This overfitting scenario in HPO has also been justified in~\cite{zhang2023targeted}.

\subsubsection{Tuning neural networks}
\label{sec:image}
We perform a neural network tuning task on a large image classification dataset Yearbook from the Wild-Time benchmark~\cite{yao2022wildtime}, which  consists of 33,431 American high school yearbook photos. Due to the change of social norms, and other potential factors that may change with the passage of time, there exist temporal distribution shifts in it~\cite{ginosar2015century}.   

To make our evaluation more comprehensive and convincing, in addition to the single-objective HPO baselines, we also include the state-of-the-art robust training methods that are applicable to this task. For each type of method mentioned in Wild-Time, we choose one algorithm with the best average test performance according to the benchmarked results. Specifically, we include the classic supervised learning method ERM, a continual learning method Fine-tuning, temporal invariant learning method LISA~\cite{yao2022improving}, a contrastive learning method SimCLR~\cite{chen2020simple} and a Bayesian learning method SWA~ \cite{izmailov2018averaging}. 
We use the implementations for those methods from Wild-Time and follow the same Eval-Fix evaluation setting with the benchmark. In addition to YearBook, we also conducted experiments on other datasets included in the Wild-Time benchmark~\cite{yao2022wildtime} as shown in Appendix~\ref{app:more_experiments}.

Table~\ref{tab:image} shows the final test results from HyperTime and all the compared methods. In terms of average performance and the worst fold performance, we observe that HyperTime is the best one compared to others. Moreover, we also observe that the performance of the HPO algorithms (single HPO algorithms and HyperTime) are significantly better than the non-HPO methods. 
We also show the winning number for each method in Table~\ref{tab:image}, HyperTime gets the best results on 7/9 of the test folds which is significantly better than other methods. 

\begin{table*}[htb!]
\vskip -0.1in
\centering
\caption{The results of baselines and our method on the yearbook dataset. We show the \textbf{average test accuracy}, the \textbf{worst fold accuracy}, and the number of winning folds (\textbf{WN}) across 9 test folds with 3 seeds, which are denoted as Test-average, Test-worst, and Winning fold num, respectively. All the numbers are the higher the better.}
\label{tab:image}
\scalebox{0.9}{
\begin{tabular}{cccccccccc}
\toprule[1.5pt]
             & ERM & Fine-tuning & LISA & SIM-CLR & SWA & CFO & BO & HB & HyperTime \\ \hline
Test-average     & 77.74        & 79.09                & 83.45         & 74.72            & 82.60        & 83.88        &83.55 &83.83 & \textbf{84.58}      \\
Test-worst   & 65.24        & 70.09                & 70.74         & 62.69            & 71.57        & 73.05       &71.23 &70.43 & \textbf{73.91}      \\
WN & 0            & 0                    & 2             & 0                & 0            & 0           &0 &0 & \textbf{7}          \\ \bottomrule[1.5pt]
\end{tabular}}
\vskip -0.1in
\end{table*}

In summary, the effectiveness of HyperTime is evidenced by its superior performance compared to single objective HPO algorithms such as CFO, BO, and HB, as well as other state-of-the-art non-HPO methods across various tasks. Furthermore, HyperTime consistently outperforms ERM on all datasets of Wild-Time, further supporting its superiority.

\subsection{Further Investigation}
\label{subsec:investigation}
In this subsection, we conduct further investigations for our method including ablation studies and an evaluation of our method when combined with robust training methods.

\subsubsection{Ablation}

We first do a series of ablation studies aiming to provide a better understanding regarding the two important components of our method: (1) Regarding the validation sets: Does the chronological re-sampling strategy matter when constructing the validation sets in our method?  
(2) Regarding the optimization objectives:
Are there easy alternatives to achieve similarly good performance?

\begin{figure}[H]
\centering
\begin{minipage}[t]{0.49\textwidth}
\centering
\includegraphics[width=3.25cm,height=2.75cm]{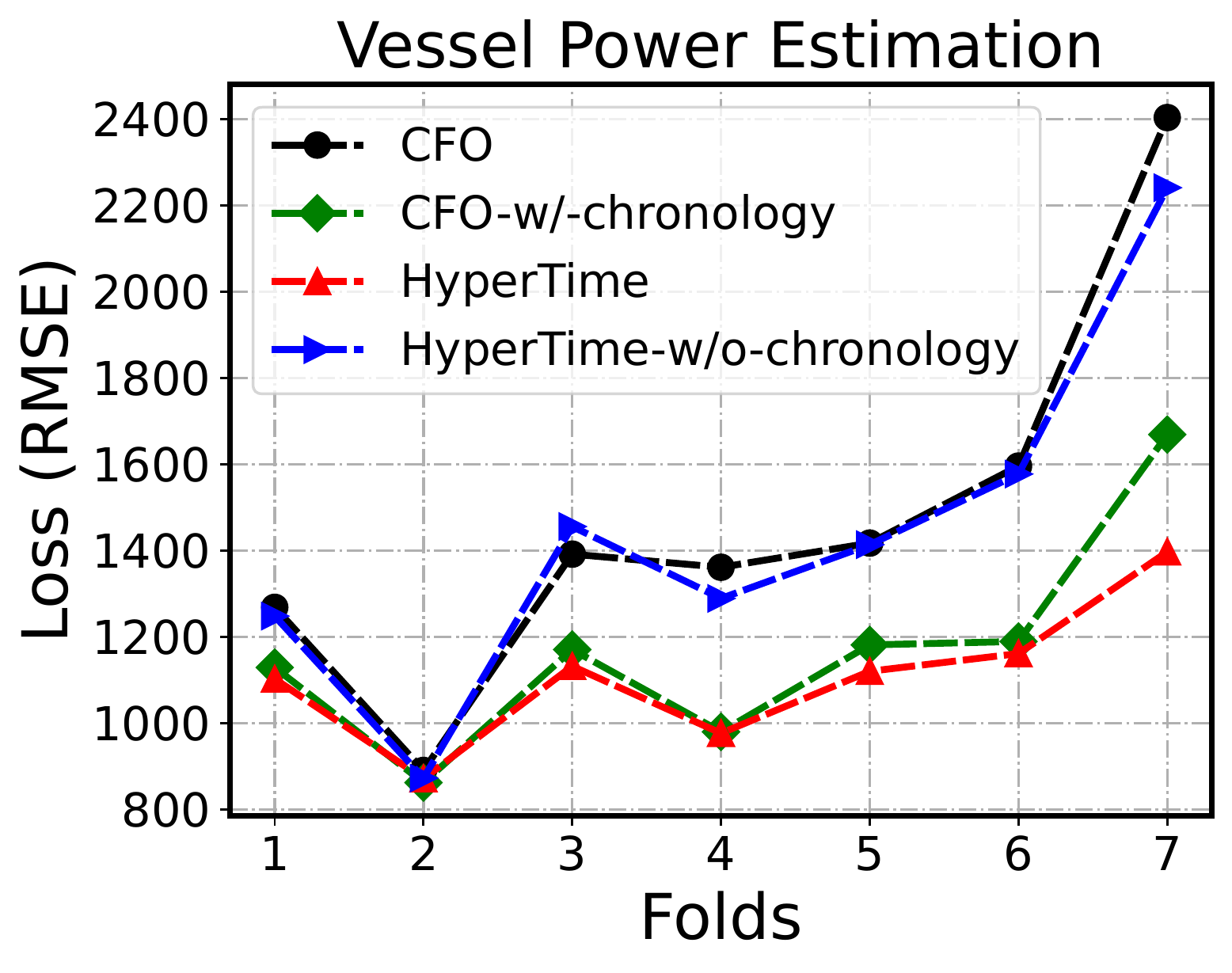}
\includegraphics[width=3.25cm,height=2.75cm]{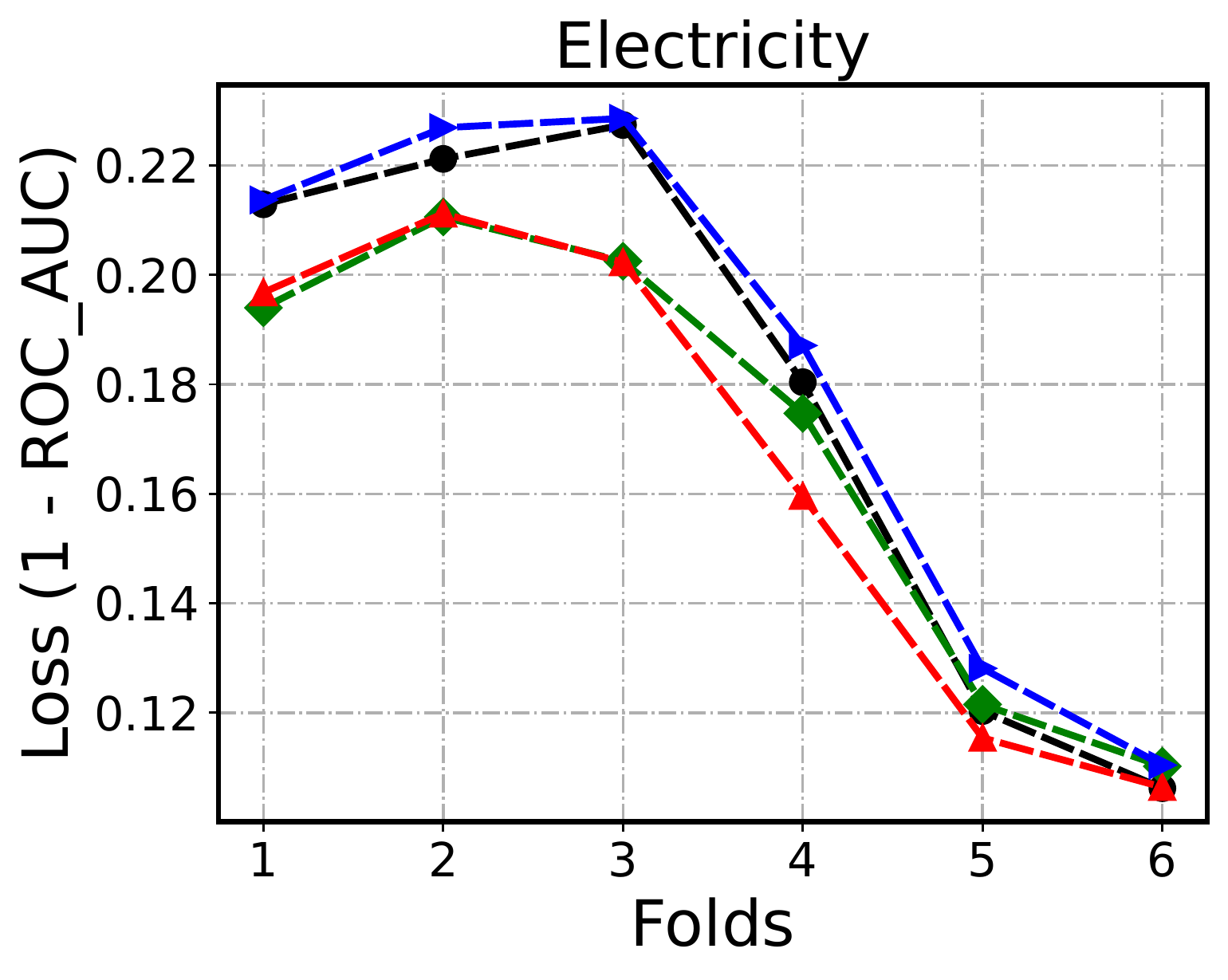}
\caption{Test loss of CFO and HyperTime on different folds with/without using chronological validation sets.}
\label{fig:abla1}
\end{minipage}
\hspace{7pt}
\begin{minipage}[t]{0.48\textwidth}
\centering
\includegraphics[width=3.25cm,height=2.75cm]{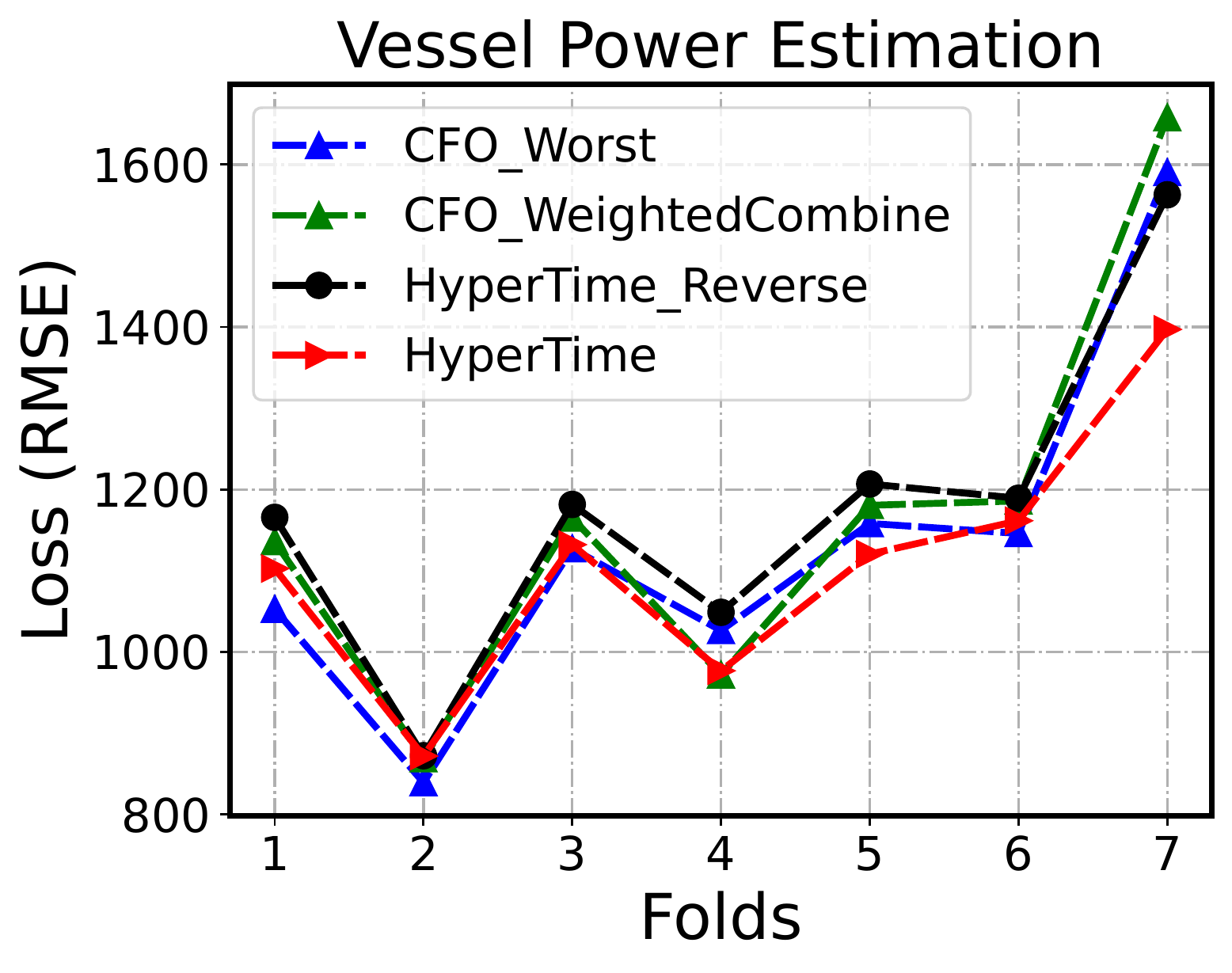}
\includegraphics[width=3.25cm,height=2.75cm]{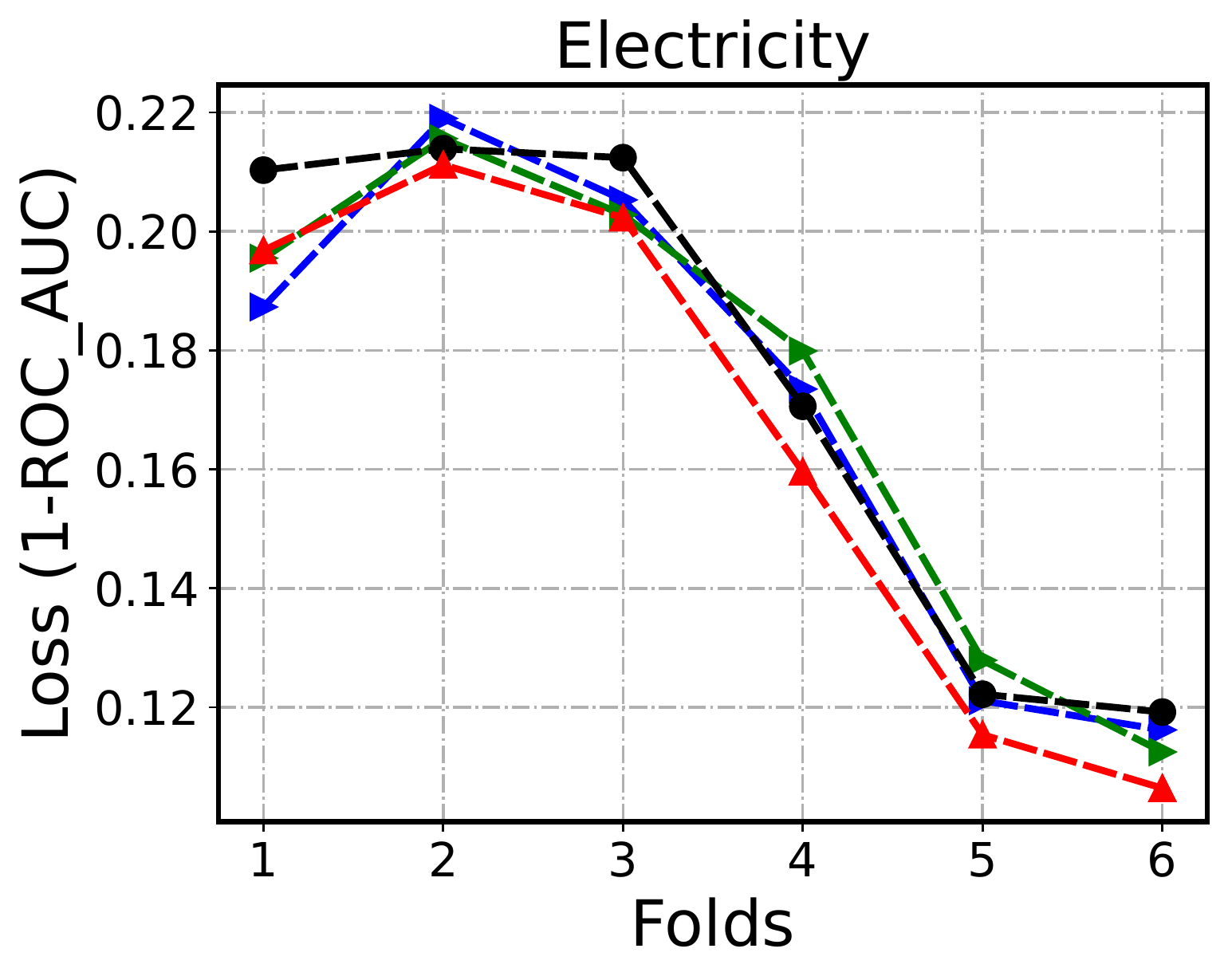}
\caption{Test loss of different folds using HyperTime, HyperTime\_Reverse, CFO\_WeightedCombine and CFO\_Worst.}
\label{fig:abla2}
\end{minipage}
\end{figure}

\paragraph{The construction of validation sets.}

We first perform experiments to investigate the validation sets construction part in HyperTime.  We construct the following two variants of CFO and HyperTime by changing the way the validation sets are constructed to study how these changes impact the final performance: 
\textbf{(1) HyperTime-w/o-chronology:} In this method, we do not use the chronologically constructed validation sets and instead construct validation sets by randomly re-sampling from shuffled datasets (no-chronological-order, conventional approach in practice).
\textbf{(2) CFO-w/-chronology:} In this method, we add the chronologically constructed validation sets in the standard CFO.

We compare the performance of (1) and (2) with their original versions, i.e., CFO and HyperTime on the Electricity dataset and Vessel power dataset. 
Figure~\ref{fig:abla1} shows the test results of these methods. 
We observe that the methods with chronological validation sets (CFO-w/-chronology and HyperTime) are obviously better than their corresponding versions with random validation sets (CFO and HyperTime-w/o-chronology).  
This indicates that chronological cross-validation is indeed an important contributing factor to the good performance of HyperTime.  

\paragraph{Optimization objectives.}

We then perform experiments to investigate the role lexicographic optimization plays in our method.
We vary optimization objective formulations in our method in different ways and investigate the factors in the objective formulations that make contributions to the final performance. 
We construct three new methods for comparison as shown below: 
 \textbf{(1) CFO\_Worst:} Using chronological validation sets and setting the worst-fold validation loss as the objective in CFO. \textbf{(2) HyperTime\_Reverse:} Reversing the priority of optimization objectives in our method, i.e., setting the worst-fold validation loss as the primary objective and the average validation across folds as the secondary objective. \textbf{(3) CFO\_WeightedCombine:} Using chronological validation sets and setting the optimization objective as a weighted combination of two objectives in CFO. Weights are 0.99 and 0.01 for average validation loss and the worst fold validation loss, respectively, which is consistent with the tolerance setting in our experiments (\(\kappa\) = 1\%).

As shown in Figure~\ref{fig:abla2}, the optimization objective formulation in our method is obviously better. There are three takeaways:
(1) HyperTime is obviously better than CFO\_Worst indicating that both two optimization objectives (average and worst fold performance) should be considered in our method.
(2) HyperTime consistently outperforms CFO\_WeightedCombine indicates that the importance of formulating the optimization of these two objectives as a lexicographic optimization problem.
(3) HyperTime consistently outperforms HyperTime\_Reverse indicating that the average validation loss shall be considered an objective of a higher priority compared with the worst-case validation loss.

\subsubsection{Compatibility with Robust Training}
\label{sec:compatibility}
\begin{table*}[htb!]
\vskip -0.1in
\centering
\caption{Test time results regarding \textbf{average test accuracy}, the \textbf{worst fold accuracy}, and the \textbf{number of winning folds} for a state-of-the-art robust training method LISA~\cite{yao2022improving}, our method HyperTime, and the methods adding LISA to CFO and HyperTime respectively. All the numbers are the higher the better.}
\label{tab:with_robust_training}
\scalebox{0.75}{
\begin{tabular}{ccccc}
\toprule[1.5pt]
                          & LISA &  CFO+LISA & HyperTime & HyperTime+LISA \\ \hline
Test-average     & 83.45                                     & 84.19             & 84.58              & \textbf{85.11}          \\
Test-worst       & 70.74                                     & 65.77             & \textbf{73.91}     & 71.90                  \\ 
Winning fold num & 0                                         & 0                 & 3                  & \textbf{6}              \\  \bottomrule[1.5pt]
\end{tabular}}
\end{table*}

Since our method is a generic hyperparameter optimization solution, it is agnostic to the specific learning method as long as there are important hyperparameters to tune. In this subsection, we show the compatibility of our method with robust training methods, which shows its advantage in further boosting the robustness of the whole machine-learning pipeline.
We perform evaluations on the yearbook dataset by adding robust optimization method LISA~\cite{yao2022improving}
to HPO which achieves the best performance in Wild-Time~\cite{yao2022wildtime}. 
Specifically, we reuse the LISA implementation from Wild-Time and use our algorithm to tune its hyperparameters, including both the architecture hyperparameters and non-architecture hyperparameters. The detailed search space is the same with Section~\ref{sec:image} as shown in Appendix~\ref{append:evaluation}. 
Table~\ref{tab:with_robust_training} shows the final overall results and we also include the test results of different folds for each method in Appendix~\ref{app:more_experiments}.
We have the observations below:  
(1) Combining HyperTime with LISA achieves better average performance compared with using either of them.
(2) Combining HyperTime with LISA has more winning numbers compared with all other methods.
(3) Combining HyperTime with LISA improves the worst fold performance over LISA, but degrades the worst fold performance compared with HyperTime alone. In summary, observations (1) and (2) demonstrate that the
combination of HyperTime and other non-HPO temporal
distribution shift solutions further boost the model
performance compared with using either of them. Observation (3) shows one disadvantage of this combination, and it is worth investigating
the reason and the method for mitigating it in future work.
\vskip -0.1in
\section{Conclusion}
\vskip -0.1in
In this work, we propose a new method to combat temporal distribution shifts named HyperTime. HyperTime approaches this problem by performing multi-objective hyperparameter tuning with a lexicographic preference across different objectives, on a set of chronologically constructed validation sets. We evaluate HyperTime across multiple datasets and learners, which verify its strong empirical performance even compared with the state-of-the-art robust training methods. Moreover, HyperTime is agnostic of learning methods, and combining it with other non-HPO robust learning methods could further boost the performance.

\newpage
\bibliographystyle{plainnat}
\bibliography{reference}
\clearpage
\appendix

\section{Details of LexiFlow}
\label{appendix_lexiflow}
LexiFlow is a randomized direct search based HPO algorithm, which is able to direct the search to the optimum based on lexicographic comparisons over pairs of configuration. 
It start from a initial hyperparameter configuration and gradually move to the optimal point by making comparisons with nearby configurations in the search space. More details about LexiFlow could be found in the paper~\cite{zhang2023targeted}.

\begin{algorithm}[H]\small
\SetNoFillComment
\SetAlCapNameFnt{\small}
\SetAlCapFnt{\small}
\caption{LexiFlow}
\label{alg:1}
\KwIn{
Objectives $\mathbf{L}(\cdot)$, tolerances $K$ (optional).
}
\textbf{Initialization:} Initial configuration $c_0$, $t' = r= s = 0$, $\delta = \delta_{init}$; \\
Obtain $\mathbf{L}(c_{0})$, and $c^* \gets c_0 $, $\cH \gets \{ c_0 \}$,  $Z_{\cH} \gets \mathbf{L}(c_0)$ \\
\While{$t = 0, 1, ...$}{
Sample $\bu$ uniformly from unit sphere $\bbS$

\lIf{\texttt{Update}($\mathbf{L}(c_t + \mu \bu)$,  $L{(c_{t})}$, $Z_{\cH}$)}{$c_{t+1} \gets c_{t} + \mu \bu$, $t' \gets t$}
\lElseIf{\texttt{Update}($\mathbf{L}(c_{t}- \mu \bu)$,  $L{(c_{t})}$, $Z_{\cH}$)}{
    $c_{t+1} \gets c_t - \mu \bu$, $t' \gets t$
}
\lElse{
    $c_{t+1} \gets c_t$, 
    $s \gets s +1$
}
$\mathcal{H} \gets \mathcal{H} \cup \{c_{t+1}\}$, and update $Z_{\cH}$ according to \eqref{eq:lexi-target-input}
\lIf{$s = {2^{d-1}}$}
{$s \gets 0, 
\delta \gets \delta \sqrt{(t'+1)/(t+1)}$
} 
\If{$\delta < \delta_{lower}$}{ \tcp{Random Restart} 
$r \gets r+1$, $c_{t+1} \gets N(c_0, I)$, $\delta \gets \delta_{init} + r$ 
}
}
~~\textbf{Procedure} \texttt{Update}($\mathbf{L}(c')$,  $\mathbf{L}{(c)}$, $Z_{\cH}$): \\
~~~~~~~~~~\If{$\mathbf{L}(c') \prec_{(Z_{\cH})} \mathbf{L}({c})$ \textbf{Or} \big($\mathbf{L}(c') =_{(Z_{\cH}) } \mathbf{L}({c})$ and $\mathbf{L}(c') \prec_l \mathbf{L}({c}\big)$)}{
~~~~~~~\If{$\mathbf{L}(c') \prec_{(Z_{\cH})} \mathbf{L}({c^*})$ \textbf{Or} \big($\mathbf{L}(c') =_{(Z_{\cH}) } \mathbf{L}({c^*})$ and $\mathbf{L}(c') \prec_l \mathbf{L}({c^*})$\big)}{~~~~~~~~$c^* \gets c'$ ~~}
~~~~~~~~~~\textbf{Return True} ~~~~~~~~~~~~~~~~~~~~}
~~\Else{\textbf{Return False} ~~~~~~~~~~~~~~~} 
\textbf{Output:} A lexi-optimal configuration $c^*$
\end{algorithm}
\footnotetext{We adjust LexiFlow and make such changes: 1. Remove the optional input targets 2. Adjust tolerance from an absolute value to a relative value in percentage.}

Given any two hyperparameter $c'$ and $c$, the targeted lexicographic relations $=_{(Z)}$, $\prec_{(Z)}$ and  $\preceq_{(Z)}$ in Algorithm~\ref{alg:1} are defined as:
\begin{align}  
 & \mathbf{L}(c')  =_{(Z)} \mathbf{L}(c)  \Leftrightarrow  L^{(i)}(c') = L^{(i)}(c)  \lor   \\ 
 &(L^{(i)}(c') \leq z^{(i)}\land L^{(i)}(c) \leq  z^{(i)}) ~~\forall i \in [1,...,I],    \nonumber \\
 &  \mathbf{L}(c')  \prec_{(Z)}  \mathbf{L}(c)   \Leftrightarrow  
  \exists i \in [I]: L^{(i)}(c') < L^{(i)}(c)  \land  \\
  & L^{(i)}(c) > z^{(i)} \land L_{i-1}(c) =_{(Z)} L_{i-1}(c'), \nonumber 
\\ 
 &  \mathbf{L}(c') \preceq_{(Z)} \mathbf{L}(c)   \Leftrightarrow  \mathbf{L}(c') \prec_{(Z)} \mathbf{L}(c) \lor  \mathbf{L}(c') =_{(Z)} \mathbf{L}(c),
\end{align}
Where $L_{i-1}(c)$ denotes the a vector with the first $i-1$ dimensions of $\mathbf{L}(c)$, i.e., $L_{i-1}(c) = [L^{(1)}(c),...,L^{(i-1)}(c)]$.
$\forall i \in [1,...,I]$, $z^{(i)}$ are computed based on historically evaluated points $\mathcal{H}$. 
$C_{\mathcal{H}}^{0} = \mathcal{H}$, $\forall i \in [1,...,I]$:
\begin{equation}
\begin{aligned} 
\label{eq:lexi-target-input}
z^{(i)} = L_{\mathcal{H}}^{(i)}*(1+\kappa^{(i)}), C_{\mathcal{H}}^{i} \coloneqq 
\{c \in C_{\mathcal{H}}^{i-1}|L^{(i)}(c)\leq z^{(i)}\}, L_{\mathcal{H}}^{(i)}  \coloneqq \min_{c \in C_{\mathcal{H}}^{i-1}} L^{(i)}(c). \\
\end{aligned}
\end{equation}
\section{Theoretical Analysis}
\label{appendix:Theoretical}
\begin{itemize}
 \item We denote the $k$-th validation set as: $\cD_1$,  $\cD_2$, ..., $\cD_K$.
\item We use $d$ to denote the a data instance pair $(\bx,y)$ in a particular validation set $\cD$ in general.
\item We use $l_{c}(d)$ to denote the loss of a particular ML model configured $c$ on data instance $d$.
\item We use \(\mathbb{P}\) to denote the test data distribution.
\end{itemize}

\begin{lem} \label{lemma:difference_between_empirical_and_expected}
   When $\cD_{\test}$ and $\cD_{\val}$ are from the same distribution, then for any $c \in \cC$, with probability at least \(1-\epsilon\) (\(\epsilon \in (0,1)\)), we have:
   \[
   |\Loss(f_{c}, \cD_{\val}) - \bE[\Loss(f_{c}, \cD_{\test})] | \leq \sqrt{\frac{\beta \ln(1/\epsilon)}{2|\cD_{val}|}},
   \]
   in which \(\beta\) is the distance between the largest and the lowest loss value on any data instance. 
\end{lem}
\begin{proof}[Proof of Lemma~\ref{lemma:difference_between_empirical_and_expected}]
We denote by \(\mathbb{P}\) the data distribution on which \(\cD_{\test}\) and \(\cD_{\val}\) is drawn from. Without loss of generality, we assume the loss function is Mean squared Error, i.e., for any validation set $\cD$, $Loss(f_{c}, \cD) = \frac{1}{|\cD|}\sum_{d \in \cD} l_{c}(d) =  \frac{1}{|\cD|}\sum_{(\bx,y) \in \cD} (f_{c}(\bx)  - y)^2$.  
We further assume a bounded loss: $\forall d \sim \mathbb{P}$, $l_{c}(d) < \beta $.  
We have:
\[
|\Loss(f_{c}, \cD_{\val}) - \bE[\Loss(f_{c}, \cD_{\test})] | = |\frac{1}{|\cD_{\val}|}\sum_{i=1}^{|\cD_{\val}|}l_{c}(d_{i}) - \bE_{d \sim \mathbb{P}}[l_{c}(d)]|,
\]
where $d_{i}$ is the \(i\)-th data instance in \( \cD_{\val}\) and thus $d_{i} \sim \mathbb{P}$.

According to Hoeffding’s inequality~\cite{hoeffding1994probability}, we have:
\begin{align}
\label{bound_sup_0}
Pr(|\Loss(f_{c}, \cD_{\val}) - \bE[\Loss(f_{c}, \cD_{\test})] |> \epsilon ) &= |\frac{1}{|\cD_{\val}|}\sum_{i=1}^{|\cD_{\val}|}l_{c}(d_{i}) - \bE_{d \sim \mathbb{P}}[l_{c}(d)]| \\ \nonumber
& \leq 2\exp{\frac{-2|\cD_{\val}| \epsilon^2}{\frac{1}{|\cD_{\val}|}\sum_{i=1}^{|\cD_{\val}|} \beta}} = 2\exp{\frac{-2|\cD_{\val}| \epsilon^2}{\beta}}. 
\end{align}

By letting \(2\exp{\frac{-2|\cD_{\val}| \epsilon^2}{\beta}} = \epsilon\), we have with probability at least \(1-\epsilon\).  
\[
|\Loss(f_{c}, \cD_{\val}) - \bE[\Loss(f_{c}, \cD_{\test})] | \leq  \sqrt{\frac{\beta \ln (2/\epsilon) }{2|\cD_{\val}|}}.
\]
Which completes the proof.
\end{proof}

\begin{proof}[Proof of Theorem \ref{theorem_bound}]
\label{proof-theorem1}
\textbf{Proof sketch.}
We consider the following two cases: (I) \(L_{k^*}(\hat c) \leq L_{avg}(\hat c)\); (II) \(L_{k^*}(\hat c) > L_{avg}(\hat c)\).

It is easy to prove that under case (I), we have \(L_{k^*}(\hat c)  \leq  L_{\avg}(\hat c) \leq (1+\kappa) L_{\avg}^* \leq (1+\kappa) L_{\avg}(c_{k^*}^*)\), in which the last inequality is based on the definition of \(L_{\avg}^*\).

Under case (II), when \( \kappa \geq \frac{L_{\avg}(c^*_{k^*})}{L^*_{\avg}} - 1\), we have $c^*_{k^*} \in \cC_*^{(0)}$, i.e., \(c^*_{k^*}\) is within the \(\kappa^{(0)}\)-tolerance from the best average validation loss, and thus
\[
  L_{k^*}(\hat c) \leq  L_{\worst}(\hat c) \leq  L_{\worst}(c_{k^*}^*),
\]
in which the last inequality is based on the fact that \(c^*_{k^*}\) is within the \(\kappa^{(0)}\)-tolerance from the best average validation loss. We choose the best configuration according to the performance on $L_{\worst}$, and thus $L_{\worst}(\hat c) < L_{\worst}(\tilde c) $ for all $\tilde c \in \cC_{^*}^{(0)}$.

Combining the conclusions under both cases and high probability concentration of \(L_{k^*}(\hat c)\) to \(\mathbb{E}[\Loss(f_{\hat c}, \cD_{\test})]\) finishes the proof. 
We provide a more rigorous proof below.

\textbf{Case 1: \(L_{k^*}(\hat c) \leq L_{avg}(\hat c)\)}.

In this case, we have,
\begin{align}
    L_{k^*}(\hat c) & \leq  L_{\avg}(\hat c) \leq (1+\kappa) L_{\avg}^* \leq (1+\kappa) L_{\avg}(c_{k^*}^*). \\ \nonumber
\end{align} 

\textbf{Case 2:} \(L_{k^*}(\hat c) > L_{avg}(\hat c)\):

When \( \kappa \geq \frac{L_{\avg}(c^*_{k^*})}{L^*_{\avg}} - 1\), we have $c^*_{k^*} \in \cC_*^1$, and thus
\begin{align}
     L_{k^*}(\hat c) & \leq  L_{\worst}(\hat c) \leq  L_{\worst}(c_{k^*}^*),
\end{align}
in which the last inequality is based on the fact that when $c^*_{k^*} \in \cC_{*}^1$, we choose the best configuration according to the performance on $L_{\worst}$, and thus $L_{\worst}(\hat c) < L_{\worst}(\tilde c) $ for all $\tilde c \in \cC_{^*}^1$.

Combining \textbf{Case 1} and \textbf{Case 2}, we have,

\begin{align} \label{eq:val_loss_bound}
   & L_{k^*}(\hat c) \leq  
\begin{cases}
  (1+\kappa)  L_{avg}(c_{k^*}^*),~~\text{if~~}   L_{k^*}(\hat c) \leq L_{avg}(\hat c)    \\
    L_{\worst}(c_{k^*}^*),~~\text{Otherwise}  
\end{cases}
\end{align}

According to the conclusion from Lemma~\ref{lemma:difference_between_empirical_and_expected}, we have,
\begin{align} \label{eq:concentration}
    \mathbb{E}[\Loss_{\cD_{\test}}(\hat c)] \leq L_{k^*}(\hat c) + \sqrt{\frac{\beta \ln(1/\delta)}{2|\cD_{val}|}}.
\end{align}

Combining Eq.~\eqref{eq:val_loss_bound} and Eq.~\eqref{eq:concentration} finishes the proof.
\end{proof}

\section{Additional Empirical Results}
\label{app:more_experiments}

\textbf{Per fold performance in Table~\ref{tab:with_robust_training}.}~In Figure~\ref{fig:combine}, we present the per-fold test performance for different methods in Table~\ref{tab:with_robust_training}. Our observations indicate that the combination of HyperTime and Lisa achieves the best performance compared to other methods. It demonstrates that the combination of HyperTime and other
non-HPO solutions overall further boost the model performance.

\begin{figure*}[htb!]
\begin{center}
\includegraphics[width=0.4\columnwidth]{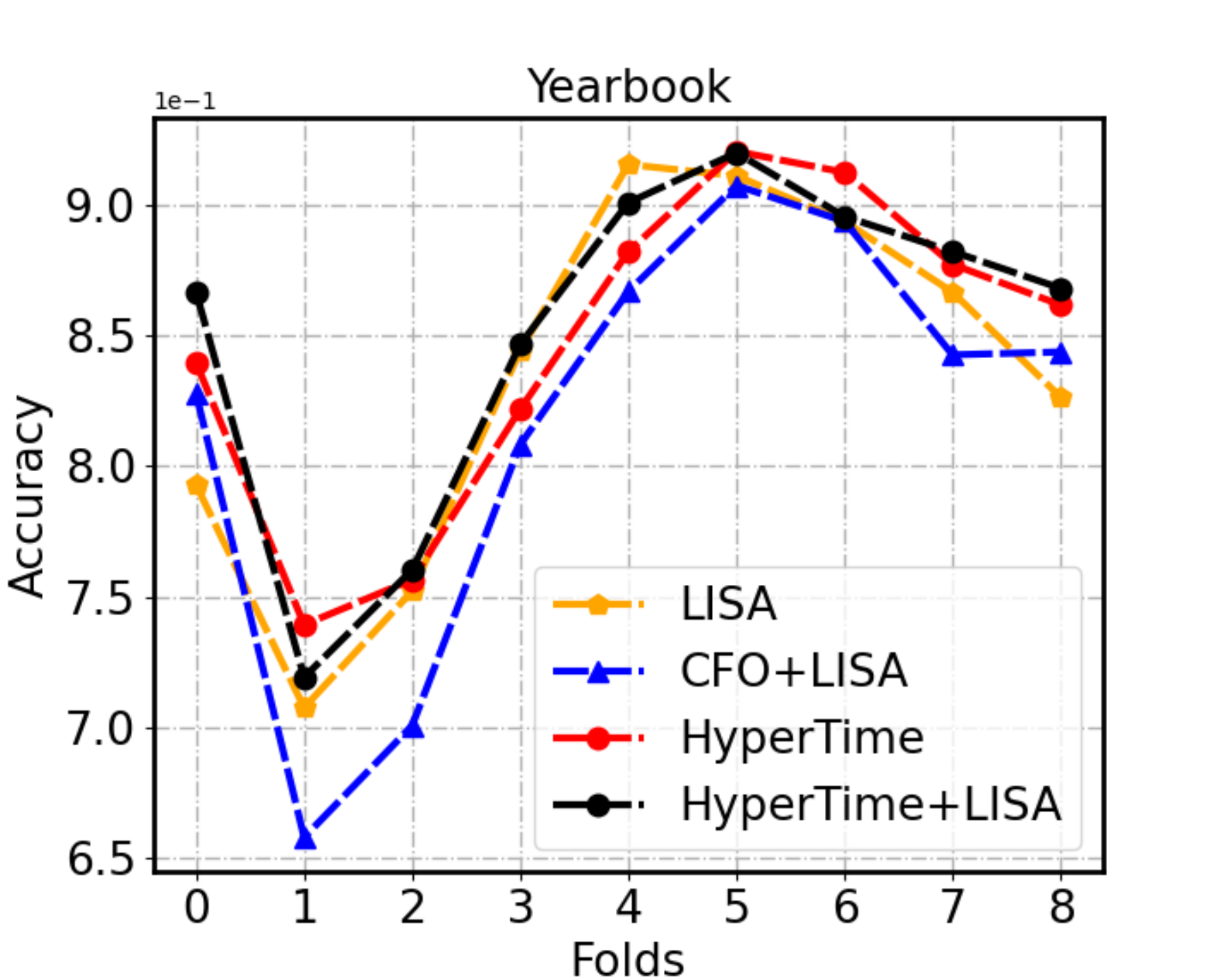}
\caption{Per fold test accuracy for a state-of-the-art robust training method LISA~\cite{yao2022improving}, our method HyperTime, and the methods combining LISA and CFO and HyperTime respectively. The results are from the same set of experiments with that in Table~\ref{tab:with_robust_training}. All the numbers are the
higher the better.}
\label{fig:combine}
\end{center}
\end{figure*}

\textbf{HyperTime is consistently better than ERM on Wild-Time~\cite{yao2022wildtime}.}~One of the main conclusions of Wild-Time~\cite{yao2022wildtime} benchmark is that there is no existing method that could consistently outperform ERM in all datasets of Wild-Time. Considering this, we compare HyperTime with ERM on all datasets of Wild-Time and we observe that HyperTime is consistently better than ERM. This demonstrates that HyperTime is a promising method compared to existing methods.
\begin{table}[h]
\caption{The comparisons between HyperTime and ERM on all datasets of Wild-Time benchmark. We show the average test performance and worst fold performance. All the numbers are the higher the better.}
\resizebox{\linewidth}{!}{
\begin{tabular}{l|ll|ll|ll|ll|ll|ll}
\toprule[1.5pt]
          & \multicolumn{2}{c|}{\textbf{MIMIC-Readmission}} & \multicolumn{2}{c|}{\textbf{MIMIC-Mortality}} & \multicolumn{2}{c|}{\textbf{HuffPost}}   & \multicolumn{2}{c|}{\textbf{Arxiv}}      & \multicolumn{2}{c|}{\textbf{FMoW-Time}}  & \multicolumn{2}{c}{\textbf{Yearbook}}    \\ 
          & Avg.       & Worst        & Avg.      & Worst       & Avg.   & Worst     & Avg.   & Worst     & Avg.   & Worst     & Avg.  & Worst     \\ \hline
ERM       & 48.02              & 43.68             & 77.24             & 73.45            & 70.60          & 69.14          & 46.39          & 44.53          & 58.05          & 46.40          & 77.74          & 65.24          \\ \hline
HyperTime & \textbf{54.81}     & \textbf{51.44}    & \textbf{78.26}    & \textbf{74.52}   & \textbf{71.68} & \textbf{69.72} & \textbf{48.48} & \textbf{46.52} & \textbf{59.17} & \textbf{50.02} & \textbf{84.58} & \textbf{73.91} \\ \bottomrule[1.5pt]
\end{tabular}}
\end{table}
\newpage

\textbf{Addition results of investigating validation sets construction.}~In addition to the test performance of each fold for different validation sets construction methods reported in Figure~\ref{fig:abla1}, we also report the average performance and the worst fold performance of different methods in Table~\ref{table:shuffled}. 

We can observe that considering both average performance and the worst fold performance, CFO and HyperTime with chronological validation sets are better compared with their corresponding versions with random validation
sets on both Electricity and Vessel power estimation. 
Moreover, HyperTime with chronological cross-validation sets achieves a better performance compared with other methods on all datasets in this experiment.
It further shows that chronological cross-validation has a dominating advantage over typical cross-validation in reaching a better loss in our method.
\begin{table*}[htb!]
\centering
\begin{center}
\caption{Test results of CFO and HyperTime using chronological and randomly shuffled folds construction methods. We show the average test accuracy and the worst fold accuracy, which are denoted as Test-average and Test-worst.}
\label{table:shuffled}
\end{center}
\small
\resizebox{0.99\textwidth}{!}{
\begin{tabular}{l|cccc|cccc}
\toprule[1.5pt]
 & \multicolumn{4}{c|}{{\color[HTML]{000000} \textbf{Electricity}}} & \multicolumn{4}{c}{\textbf{Vessel Power}} \\ \hline
Loss Type & \multicolumn{4}{c|}{1-ROC\_AUC} & \multicolumn{4}{c}{RMSE} \\ \hline
Method & CFO & \multicolumn{1}{c|}{HyperTime} & CFO & HyperTime & CFO & \multicolumn{1}{c|}{HyperTime} & {\color[HTML]{000000} CFO} & {\color[HTML]{000000} HyperTime} \\ \hline
With Chronology & True & True & False & False & True & True & False & False \\ \hline
Test-average & 0.1689 & \textbf{0.1653} & 0.1781 & 0.1825 & 1168.9258 & \textbf{1108.9676} & {\color[HTML]{000000} 1475.6870} & {\color[HTML]{000000} 1442.6700} \\
Test-worst & \textbf{0.2106} & 0.2112 & 0.2274 & 0.2286 & 1668.9744 & \textbf{1397.1399} & {\color[HTML]{000000} 2403.6068} & {\color[HTML]{000000} 2241.1207} \\ \bottomrule[1.5pt]
\end{tabular}}
\end{table*}

\textbf{Supplementary results of comparing HyperTime with CFO\_WeightedCombine}

In this section, we conduct additional experiments to compare HyperTime with CFO\_WeightedCombine, which set the optimization objectives as a weighted combination of average validation loss and the worst fold validation loss in CFO. To represent the weights assigned to the worst fold validation loss, we use the symbol $\lambda$. Consequently, we set the weight for the average validation loss as $1-\lambda$. We use four different $\lambda$ settings for CFO\_WeightedCombine: 5\%, 10\%, 15\%, and 20\%. Figure~\ref{fig:more_weightedcombine} shows the per-fold test loss of these two methods.

We observe that HyperTime outperforms CFO\_WeightedCombine under all four weight settings. This further demonstrates the importance of formulating the optimization of these two objectives as a lexicographic optimization problem.

\begin{figure*}[htb!]
\vskip -0.1in
\begin{center}
\includegraphics[width=0.24\textwidth]{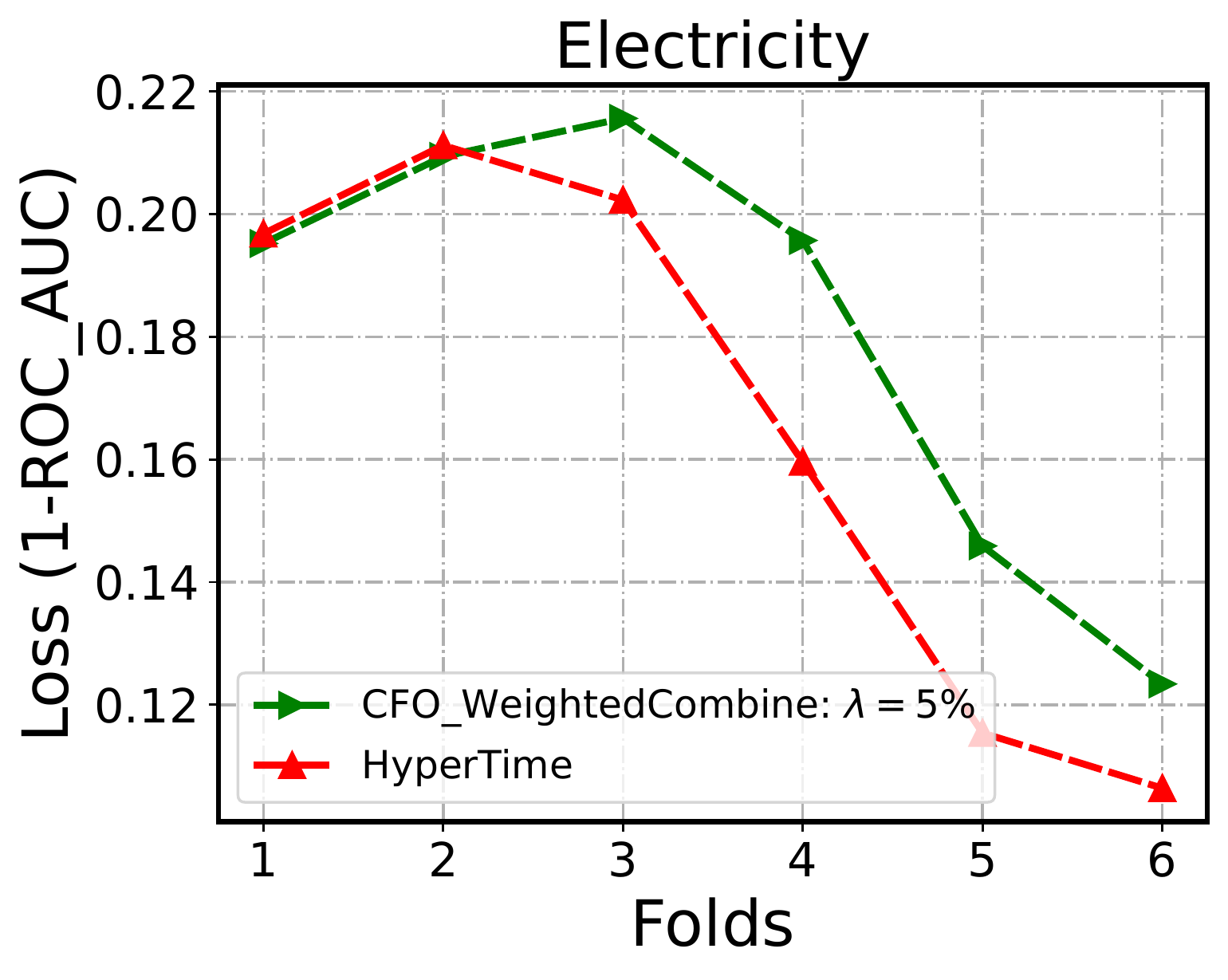}
\includegraphics[width=0.24\textwidth]{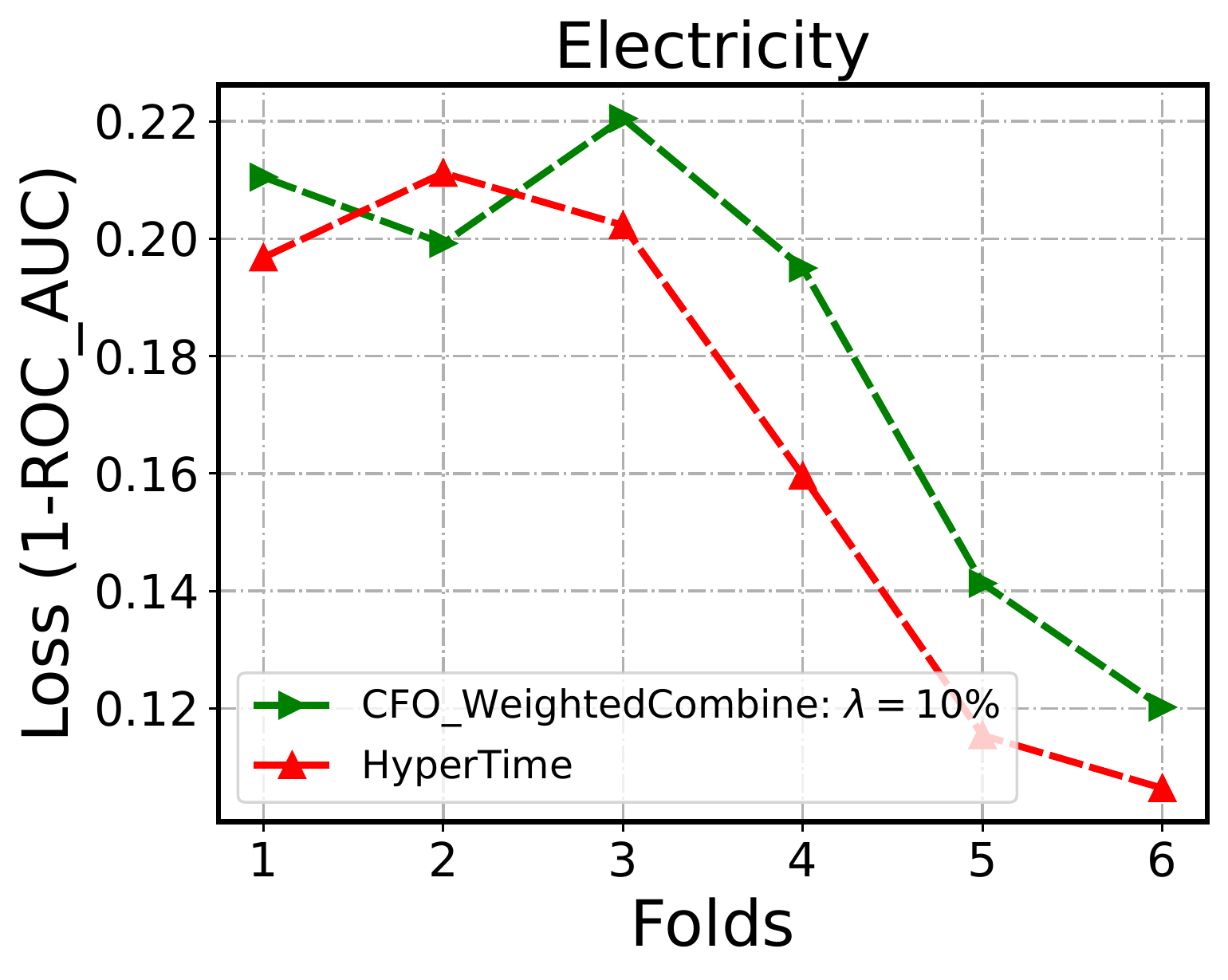}
\includegraphics[width=0.24\textwidth]{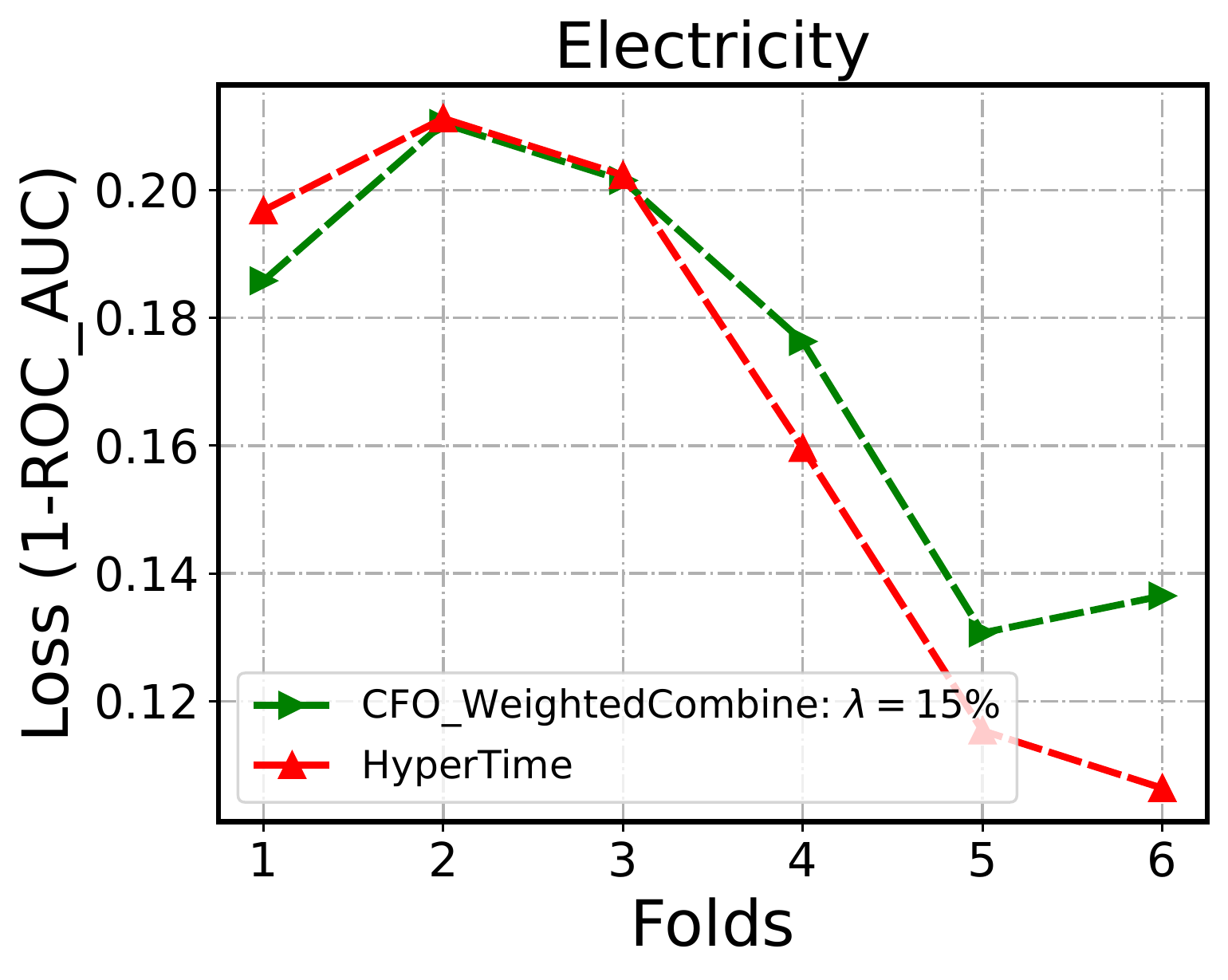}
\includegraphics[width=0.24\textwidth]{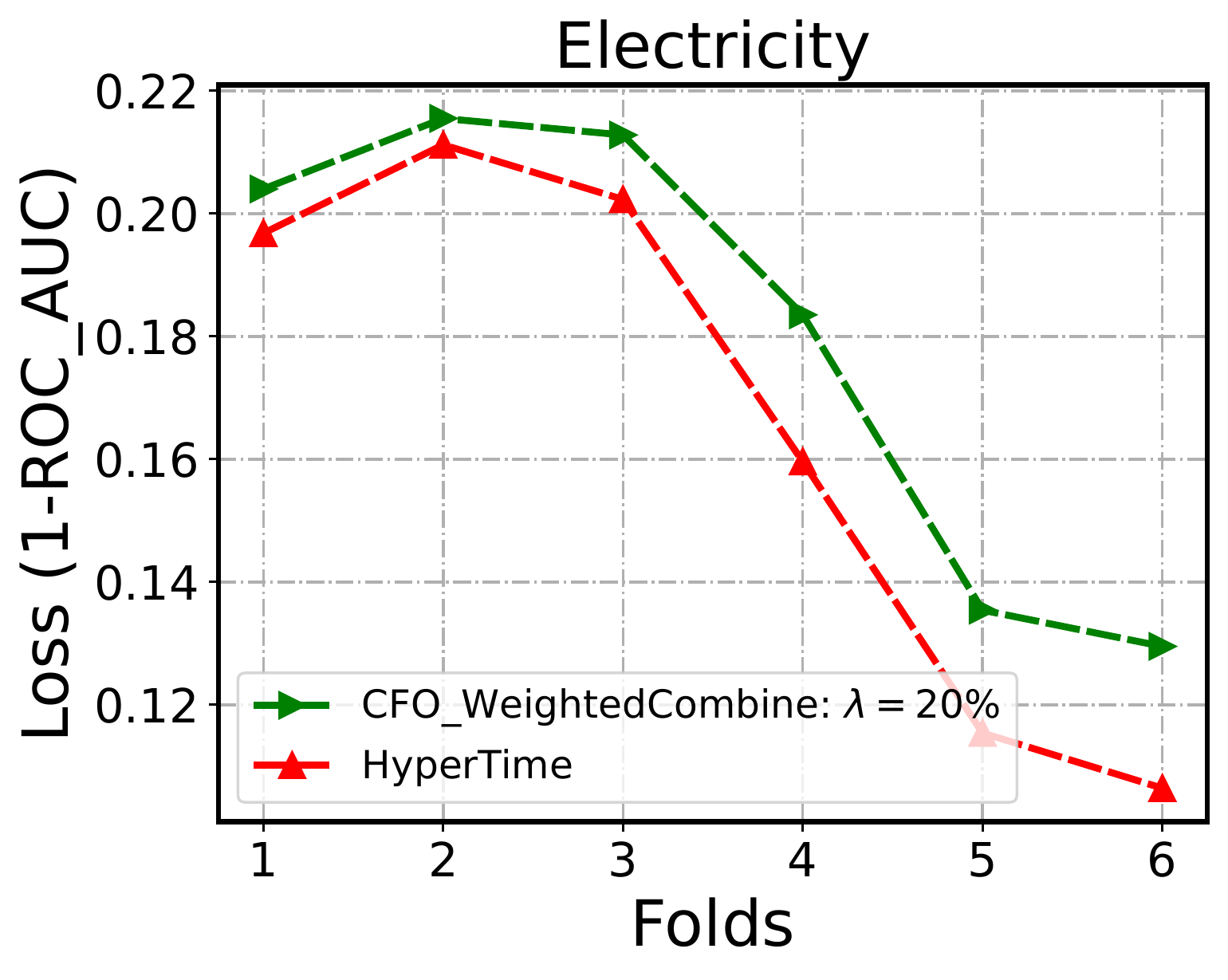}
\includegraphics[width=0.24\textwidth]{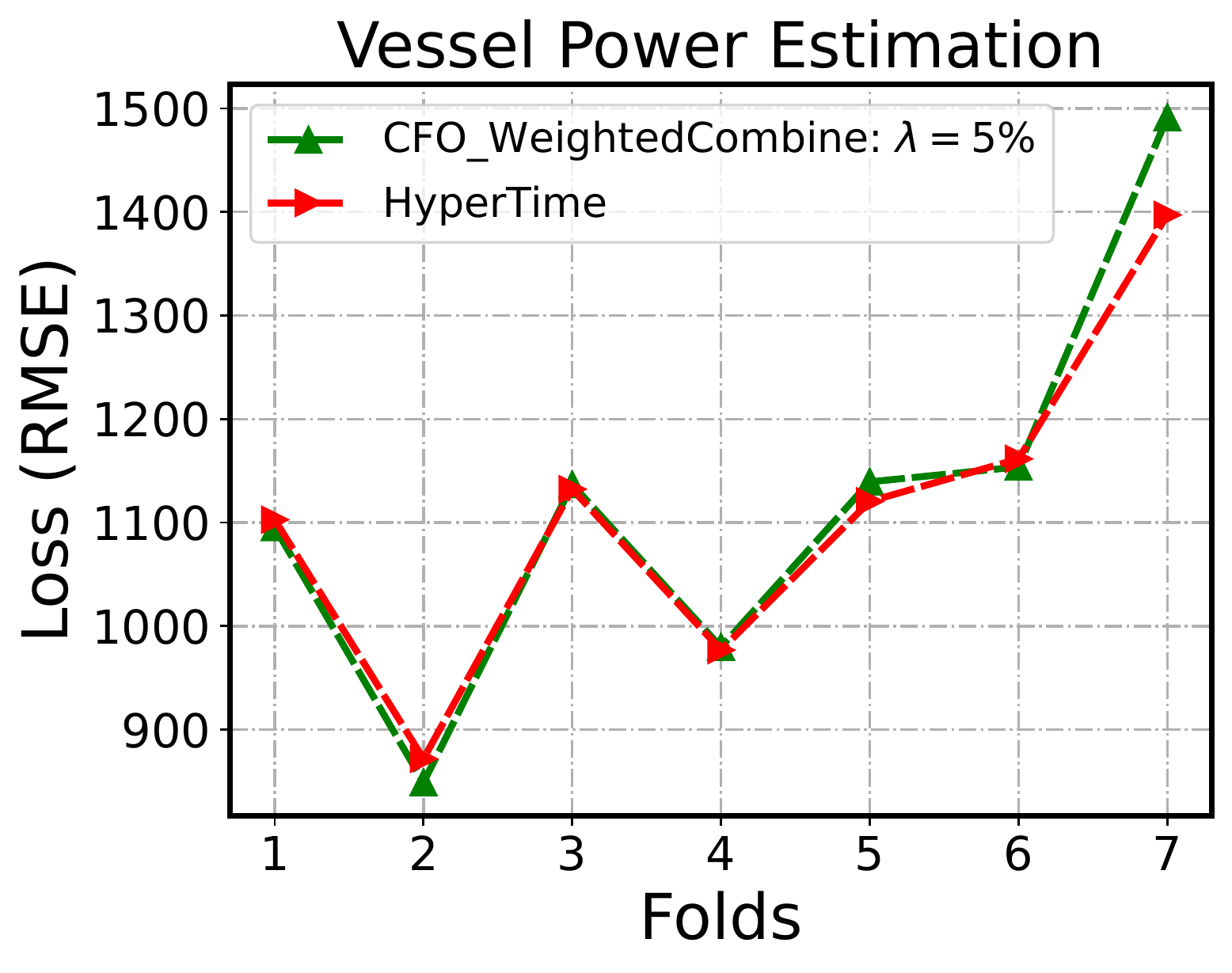}
\includegraphics[width=0.24\textwidth]{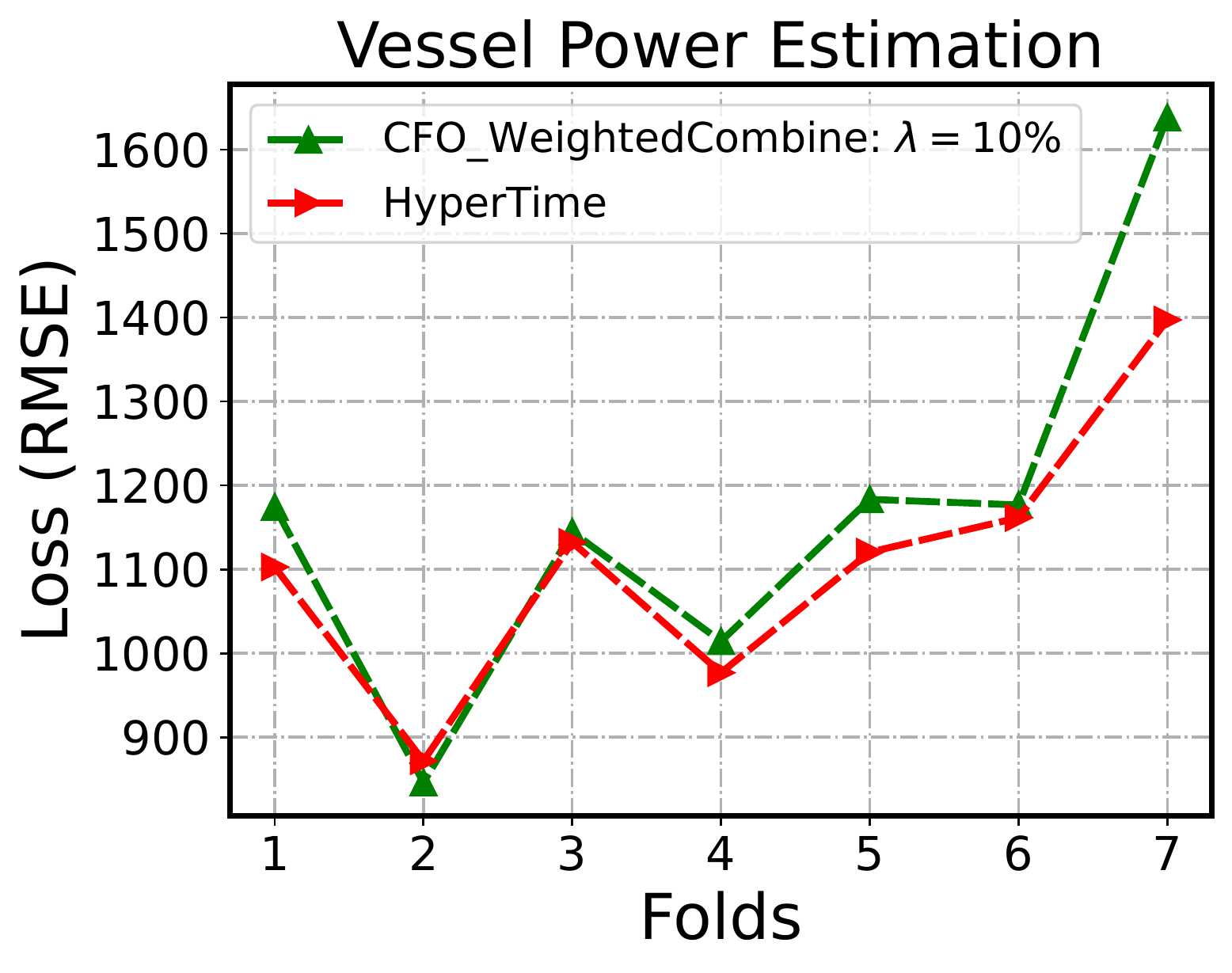}
\includegraphics[width=0.24\textwidth]{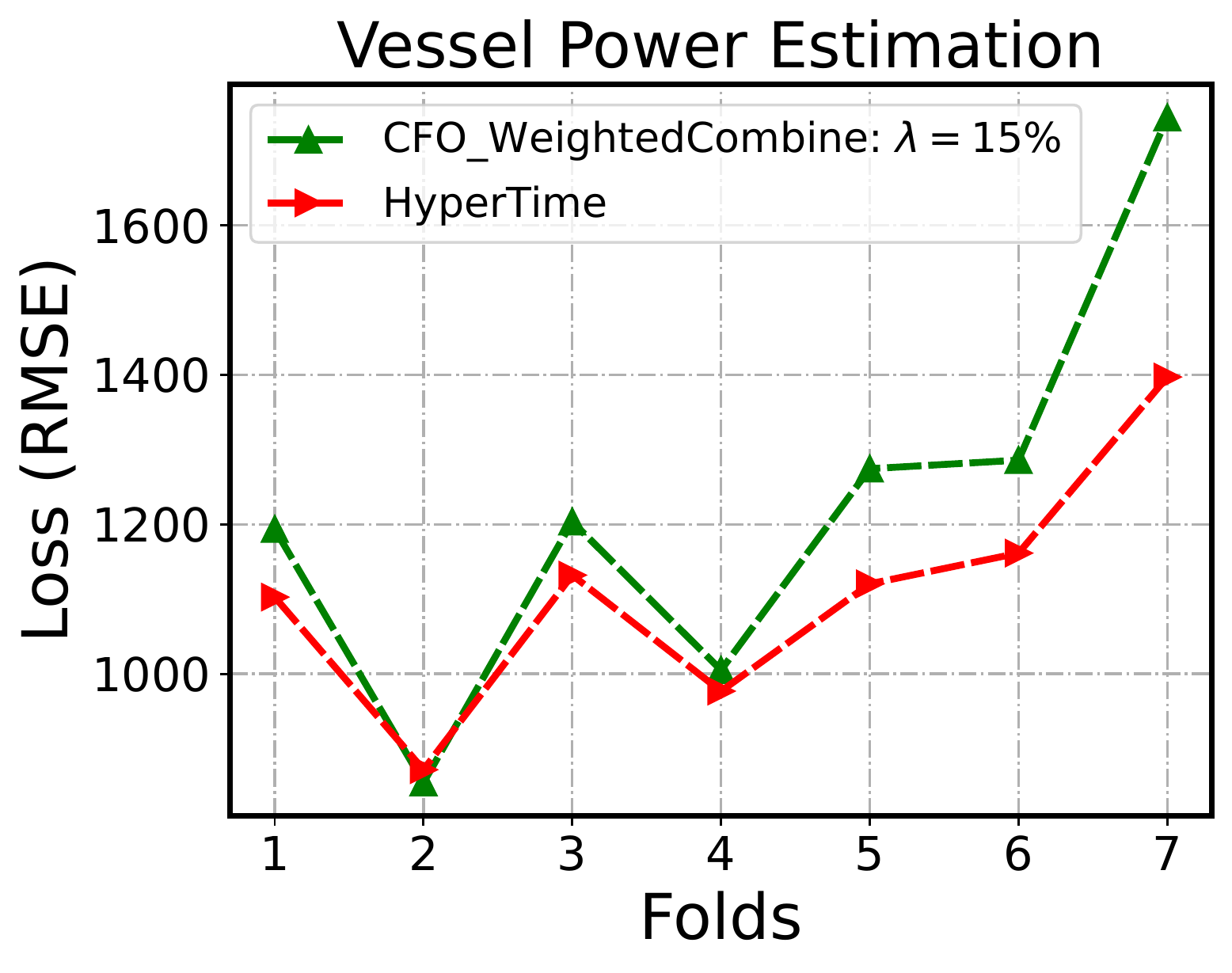}
\includegraphics[width=0.24\textwidth]{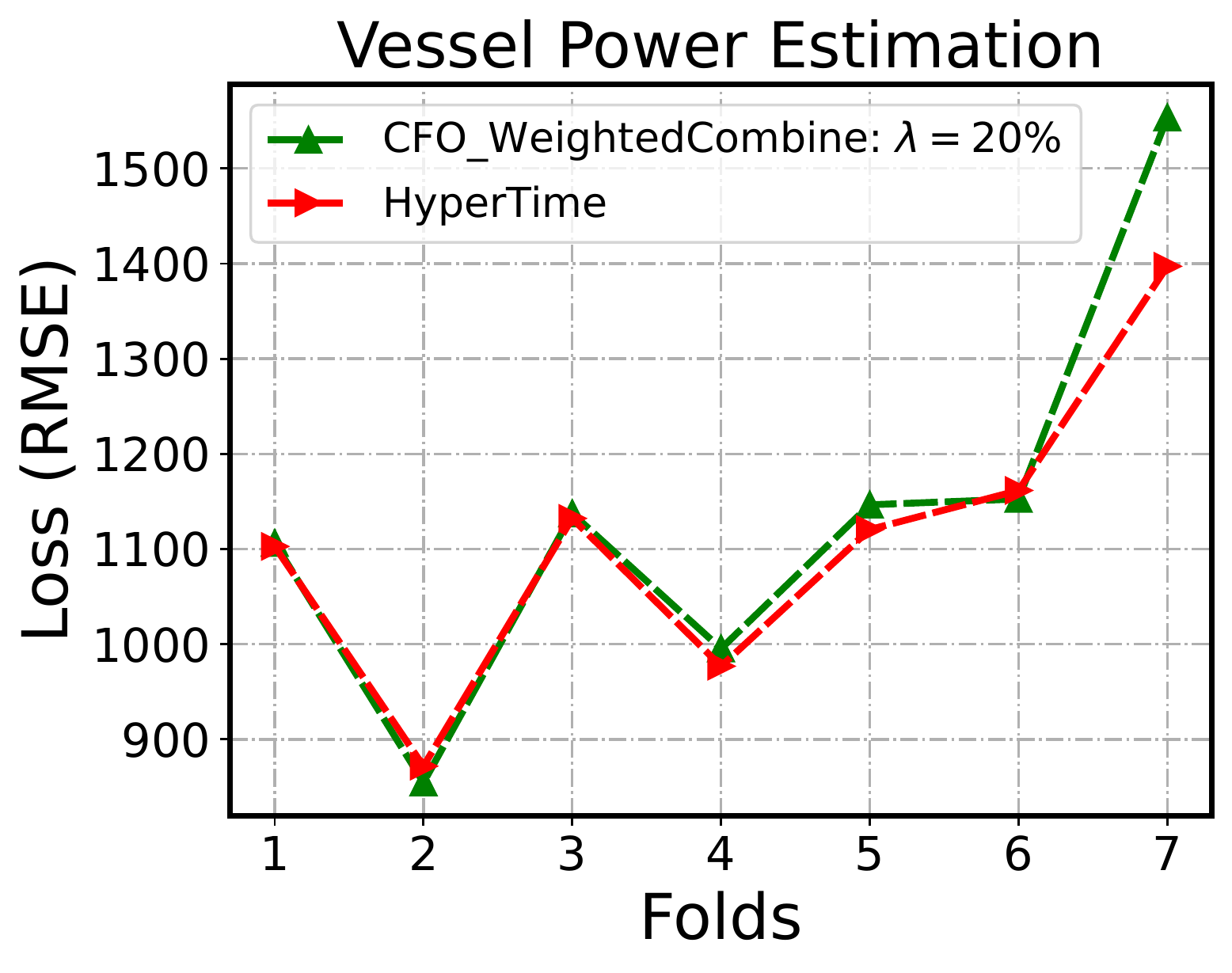}
\caption{Per fold test loss (lower the better) for Hypertime and CFO\_WeightedCombine on Electricity and Vessel Power Estimation datasets with different weight settings. The results are averaged over five random seeds.} \label{fig:more_weightedcombine}
\end{center}
\vskip -0.15in
\end{figure*}

\section{Search Space}
\label{append:evaluation}
\subsection{Search Space of gradient-boosting tree}
\begin{table}[h]
\vskip -0.15in
\caption{Hyperparameters tuned in XGboost.} 
\label{Appendix:XGboost_hyperparameters}
\centering
\scalebox{1}{
\begin{tabular} {r| l |l} 
hyperparameter  & type & range \\ 
 \hline
 estimators number & int & [4,  min(32768, train$\_$datasize)] \\
 max leaves & int & [4,  min(32768, train$\_$datasize)]  \\ 
 max depth & int & [0, 6, 12]   \\
 min child weight & float & [0.001, 128] \\
 learning rate & float & [1/1024, 1.0] \\
 subsample & float & [0.1, 1.0]  \\
 colsample by tree & float & [0.01, 1.0] \\
 colsample by level & float & [0.01, 1.0]  \\
 reg alpha & float & [1/1024, 1024]   \\
 reg lambda & float & [1/1024, 1024]  \\
\end{tabular}}
\end{table}

\begin{table}[h]
\caption{Hyperparameters tuned in LGBM.} 
\label{Appendix:LGBM_hyperparameters}
\centering
\scalebox{1}{
\begin{tabular} {r| l |l} 
hyperparameter  & type & range \\ 
 \hline
 estimators number & int & [4,  min(32768, train$\_$datasize)] \\
 leaves number & int & [4,  min(32768, train$\_$datasize)]  \\ 
 min child sample & int & [2, 129] \\
 learning rate & float & [1/1024, 1.0] \\
 log\_max\_bin & int & [3, 11]  \\
 colsample by tree & float & [0.01, 1.0] \\
 reg alpha & float & [1/1024, 1024]   \\
 reg lambda & float & [1/1024, 1024]  \\
\end{tabular}}

\end{table}

\subsection{Search Space of neural network}

We use the same neural network backbone as the Wild-Time~\cite{yao2022wildtime} benchmark for different datasets based on its source code~\footnote{https://github.com/huaxiuyao/Wild-Time}. We list the detailed search space used in different datasets in Table~\ref{table:nn_hyperparameters}.

\begin{table}[H]
    \caption{Hyperparameters tuned in neural networks.} 
    \begin{center}
    { \scalebox{0.83}{
    \begin{tabular}{llll}
        \toprule
        \textbf{Dataset} & \textbf{Hyparameter} & \textbf{type} & \textbf{Range}\\
        \midrule
        \multirow{5}{*}{Yearbook}       & Training iteration & int & $[3000, 5000]$\\
                                      & learning rate    & float   & $[$1e-4$, $1e-1$]$\\
                                      & batch size & int & $\{32,64,128,256\}$\\
                                      & n\_conv\_channels & int & $[16,512]$ \\
                                      & kernel\_size & int & $\{2,3,4,5\}$ \\
                                      & has\_max\_pool & bool & True or Flase \\
        \midrule
        \multirow{5}{*}{FMoW-Time}   & Training iteration & int & $[3000,6000]$\\
                                      & learning rate    & float   & $[$1.5e-5$,$3e-4$]$\\
                                      & batch size & int & $\{32,64,128,256\}$\\
                                      & weight\_decay & float & $[0,0.03]$ \\
        \midrule
        \multirow{2}{*}{MIMIC-IV}     & Training iteration & int & $[3000,5000]$\\
                                      & learning rate    & float   & $[$5e-4$,$5e-2$]$\\
                                      & n\_head & int & $\{2,3,4,5\}$\\
                                      & n\_layer & int & $\{2,3,4,5\}$ \\
                                      & hidden\_size & int & $\{64,128,256,512\}$ \\
        \midrule
        \multirow{2}{*}{Huffpost}    & Training iteration & int & $[6000,8000]$\\
                                      & learning rate    & float   & $[$1e-5$,$1e-4$]$\\
                                      & weight\_decay & float & $[0.01,0.03]$ \\
        \midrule
        \multirow{2}{*}{arXiv}        & Training iteration & int & $[6000,8000]$\\
                                      & learning rate    & float   & $[$1e-5$,$1e-4$]$\\
                                      & weight\_decay & float & $[0.01,0.03]$ \\
        \bottomrule
    \end{tabular}
    }}
    \end{center}
    \label{table:nn_hyperparameters}
\end{table}

\section{Dataset Details}
\label{appendix:dataset}
\subsection{Datasets used in tuning tree-based boosting methods}

\textbf{Overall Informations}

\begin{itemize}[leftmargin=*]
\vspace{-4mm}
\setlength\itemsep{-0.2em}
\item Electricity: A classification task. It is widely used for evaluating distribution shifts mitigation methods~\cite{mallick2022matchmaker}. 
The dataset contains two and a half years of data. We exclude the first half year and use the next one year for training and the last year for testing. We split every 2 months into one fold. 
\item Vessel power estimation: A regression task taken from Wild-Time benchmark~\cite{malinin2023shifts}. It is a large dataset with 523,190 training samples over 4 years, and we use the out-of-distribution dev-set as our test data which has 18,108 samples. We split the training data uniformly into 12 folds, and the test data into 7 folds.
\item Urban temperature prediction: A regression task to predict the urban daily maximum of average 2-m temperature. It has distribution shifts as mentioned in~\cite{KAYETAL15,OLESONETAL18}. We split every 5 years into one fold and we use the first 40 years for training and test on the remaining 35 years. 
\vspace{-4mm}
\end{itemize}

We first show the overall information of each dataset in our experiment including the feature number, instance number, and the number of validation and test folds.

\begin{table*}[htb]
\centering
\caption{The feature number, instance number, and validation/test folds number of each dataset in the paper. }\label{table:datasets}
\scalebox{0.75}{
\begin{tabular}{c|c|c|c|c}
\toprule[1.5pt]
 & \textbf{Feature num} & \multicolumn{1}{l|}{\textbf{Instance num}} & \textbf{Val fold num} & \textbf{Test Fold num} \\ \hline
\textbf{Electricity} & 8 & 33873 & 6 & 6 \\
\textbf{Vessel power estimation}  &  11 & 541298 & 12 & 7 \\
\multicolumn{1}{l|}{\textbf{Temperature prediction}} & 10 & 437884 & 8 & 7 \\
\textbf{YearBook} & \textbackslash{} &  33431 & 8 & 9  \\ \bottomrule[1.5pt]
\end{tabular}}
\end{table*}

\textbf{More details about Temperature prediction dataset}

Temperature prediction is a synthetic data for urban climate research, which includes 75 years of urban climate condition information in specific areas. It has distribution shifts as mentioned in existing urban climate research works~\cite{KAYETAL15,OLESONETAL18}.
Here we select 16 gridcells of data according to ~\cite{ZHENGETAL21}, with the latidude of \(35.34, 36.28, 37.23, 38.17\) and longitude of \(115.0, 116.2, 117.5, 118.8\).
It includes ten features including near-surface humidity, eastward near-surface wind, precipitation, etc.  
In our experiment, we predict the urban daily maximum of average 2-m temperature which could be regarded as a regression task. 
More information about this data is available at ~\cite{ZHENGETAL21}.

\subsection{The datasets in the Wild-Time benchmark}

\textbf{Overall Informations}

\begin{itemize}[leftmargin=*]
\item \textbf{Yearbook}: Yearbook is an image dataset with 37,921 frontal-facing American high school yearbook photos from 1930 - 2013. Each data point is a 32 $\times$ 32 $\times$ 1 grey-scale image and the label is the student’s gender. 
Distribution shifts occur due to social norms, fashion styles, and population demographics changing over time. Following the same setting with Wild-Time, we use 1970 as the split timestep to split the training and test set. 
\item \textbf{FMoW-Time}: FMow-Time (Functional Map of the world) is a satellite imagery dataset that consists of 141,696 examples from 2002 - 2017. Each input of one data instance is a 224 $\times$ 224 RGB satellite image, and the corresponding label is one of 62 land use categories. Due to human activity, satellite imagery changes over time which is a kind of temporal distribution shift. Following the same setting with Wild-Time, we use the data from 2002 - 2012 as the training set and use the data from 2013 -2017 as the test set. 
\item \textbf{MIMIC-IV}: MIMIC-IV is one of the largest public healthcare datasets that consists of a vast number of medical records of over 40,000 patients. In this dataset, a temporal distribution shift happens over time considering the emergence of new treatments and changes in patient demographics.   
In our experiments, we treat each admission as one
record, resulting in 216,487 healthcare records from 2008 - 2019. Following the same setting as Wild-Time, the training set is from 2008-2013 and the test set is from 2014 - 2020.
We consider two classification tasks:
(1) MIMIC-Readmission aims to predict the risk of being readmitted to the hospital within 15 days. (2) MIMIC-Mortality aims to predict in-hospital mortality for each patient.
\item \textbf{Huffpost}: The task of the HuffPost dataset is to identify tags of news articles from their headlines. Temporal distribution shifts occur over time due to changes in the style or content of current events. 
For each data instance, the input feature is a news headline and the output is the news categories. In our experiment, we only include the categories that appear in all years from 2012 - 2018. Following the same setting with Wild-Time, we use 2012 - 2015 as the training set while 2016 - 2018 as the test set. 
\item \textbf{arXiv}: The task of arXiv dataset is to predict the primary category of arViv pre-prints given the paper title as input. Temporal distribution shifts occur due to the evolution of research fields. This dataset includes 172 pre-print categories from 2007 - 2022. Following the same setting with Wild-Time, we use 2007 - 2016 as the training set while 2017 - 2022 as the test set.
\end{itemize}

\textbf{More information about fold splitting}

To ensure fair and consistent comparisons, we use the same validation/test folds splitting setting as the Wild-Time benchmark. We list the number of training and test folds in Table~\ref{table: fold_information} below. More information can be found in \cite{yao2022wildtime}.

\begin{table*}[htb]
\centering
\caption{The number of validation folds and test folds for each dataset of the Wild-Time benchmark in our experiments.}\label{table: fold_information}
\scalebox{0.99}{
\begin{tabular}{c|c|c|c|c|c}
\toprule[1.5pt]
 & \textbf{Yearbook} & \multicolumn{1}{l|}{\textbf{FMoW-Time}} & \textbf{MIMIC-IV} & \textbf{Huffpost} & \textbf{arXiv} \\ \hline
\textbf{Val fold num} & 8 & 11  & 2  & 4  & 10 \\
\textbf{Test fold num}  &  9 & 5 &2  &3 &6 \\
\bottomrule[1.5pt]
\end{tabular}}
\end{table*}

\end{document}

%% file: macro-math.tex
\allowdisplaybreaks

\newcommand{\bx}{{\mathbf{x}}}

\newcommand{\bE}{{\mathbb E}}

\newcommand{\cC}{{\mathcal C}}
\newcommand{\cH}{{\mathcal H}}
\newcommand{\cD}{{\mathcal D}}

\newcommand{\train}{\text{train}\xspace}
\newcommand{\test}{{\text{test}}\xspace}
\newcommand{\val}{{\text{val}}\xspace}
\newcommand{\Loss}{{\texttt{Loss}}}
\newcommand{\worst}{\text{worst}\xspace}
\newcommand{\avg}{\text{avg}\xspace}
\newcommand{\LexiMin}{\text{LexiMin}\xspace}
\definecolor{purple}{rgb}{0.8,0.0,0.8}

\def \bu {\mathbf{u}}
\def \bbS {\mathbb{S}}

\DeclareMathOperator*{\argmin}{arg\,min}

\newtheorem{thm}{Theorem}
\newtheorem{lem}[thm]{Lemma}

\theoremstyle{definition}

\theoremstyle{plain}
\newtheorem{theorem}{Theorem}[section]

\theoremstyle{definition}

\theoremstyle{remark}
\newtheorem{remark}[theorem]{Remark}

%% file: main.bbl
\begin{thebibliography}{50}
\providecommand{\natexlab}[1]{#1}
\providecommand{\url}[1]{\texttt{#1}}
\expandafter\ifx\csname urlstyle\endcsname\relax
  \providecommand{\doi}[1]{doi: #1}\else
  \providecommand{\doi}{doi: \begingroup \urlstyle{rm}\Url}\fi

\bibitem[Abu-Mostafa et~al.(2012)Abu-Mostafa, Magdon-Ismail, and
  Lin]{learning_from_data_book}
Yaser~S. Abu-Mostafa, Malik Magdon-Ismail, and Hsuan-Tien Lin.
\newblock \emph{Learning From Data}.
\newblock AMLBook, 2012.
\newblock ISBN 1600490069.

\bibitem[Adel et~al.(2019)Adel, Zhao, and Turner]{adel2019continual}
Tameem Adel, Han Zhao, and Richard~E Turner.
\newblock Continual learning with adaptive weights (claw).
\newblock \emph{arXiv preprint arXiv:1911.09514}, 2019.

\bibitem[Arjovsky et~al.(2019)Arjovsky, Bottou, Gulrajani, and
  Lopez-Paz]{arjovsky2019invariant}
Martin Arjovsky, L{\'e}on Bottou, Ishaan Gulrajani, and David Lopez-Paz.
\newblock Invariant risk minimization.
\newblock \emph{arXiv preprint arXiv:1907.02893}, 2019.

\bibitem[Bai et~al.(2021)Bai, Zhou, Hong, Ye, Chan, and Li]{bai2021ood}
Haoyue Bai, Fengwei Zhou, Lanqing Hong, Nanyang Ye, S-H~Gary Chan, and Zhenguo
  Li.
\newblock Nas-ood: Neural architecture search for out-of-distribution
  generalization.
\newblock In \emph{Proceedings of the IEEE/CVF International Conference on
  Computer Vision}, pages 8320--8329, 2021.

\bibitem[Bent{\'e}jac et~al.(2021)Bent{\'e}jac, Cs{\"o}rg{\H{o}}, and
  Mart{\'\i}nez-Mu{\~n}oz]{bentejac2021comparative}
Candice Bent{\'e}jac, Anna Cs{\"o}rg{\H{o}}, and Gonzalo
  Mart{\'\i}nez-Mu{\~n}oz.
\newblock A comparative analysis of gradient boosting algorithms.
\newblock \emph{Artificial Intelligence Review}, 54:\penalty0 1937--1967, 2021.

\bibitem[Bergstra and Bengio(2012)]{bergstra2012random}
James Bergstra and Yoshua Bengio.
\newblock Random search for hyper-parameter optimization.
\newblock \emph{Journal of machine learning research}, 13\penalty0 (2), 2012.

\bibitem[Bergstra et~al.(2011)Bergstra, Bardenet, Bengio, and
  K{\'e}gl]{bergstra2011algorithms}
James Bergstra, R{\'e}mi Bardenet, Yoshua Bengio, and Bal{\'a}zs K{\'e}gl.
\newblock Algorithms for hyper-parameter optimization.
\newblock \emph{Advances in neural information processing systems}, 24, 2011.

\bibitem[Bertsimas et~al.(2018)Bertsimas, Gupta, and Kallus]{bertsimas2018data}
Dimitris Bertsimas, Vishal Gupta, and Nathan Kallus.
\newblock Data-driven robust optimization.
\newblock \emph{Mathematical Programming}, pages 235--292, 2018.

\bibitem[Caron et~al.(2020)Caron, Misra, Mairal, Goyal, Bojanowski, and
  Joulin]{caron2020unsupervised}
Mathilde Caron, Ishan Misra, Julien Mairal, Priya Goyal, Piotr Bojanowski, and
  Armand Joulin.
\newblock Unsupervised learning of visual features by contrasting cluster
  assignments.
\newblock \emph{Advances in Neural Information Processing Systems},
  33:\penalty0 9912--9924, 2020.

\bibitem[Chaudhry et~al.(2018)Chaudhry, Dokania, Ajanthan, and
  Torr]{chaudhry2018riemannian}
Arslan Chaudhry, Puneet~K Dokania, Thalaiyasingam Ajanthan, and Philip~HS Torr.
\newblock Riemannian walk for incremental learning: Understanding forgetting
  and intransigence.
\newblock In \emph{Proceedings of the European Conference on Computer Vision
  (ECCV)}, pages 532--547, 2018.

\bibitem[Chen et~al.(2020)Chen, Kornblith, Norouzi, and Hinton]{chen2020simple}
Ting Chen, Simon Kornblith, Mohammad Norouzi, and Geoffrey Hinton.
\newblock A simple framework for contrastive learning of visual
  representations.
\newblock In \emph{International conference on machine learning}, pages
  1597--1607. PMLR, 2020.

\bibitem[Delage and Ye(2010)]{delage2010distributionally}
Erick Delage and Yinyu Ye.
\newblock Distributionally robust optimization under moment uncertainty with
  application to data-driven problems.
\newblock \emph{Operations research}, 58\penalty0 (3):\penalty0 595--612, 2010.

\bibitem[Dong et~al.(2020)Dong, Yu, Cao, Shi, and Ma]{dong2020survey}
Xibin Dong, Zhiwen Yu, Wenming Cao, Yifan Shi, and Qianli Ma.
\newblock A survey on ensemble learning.
\newblock \emph{Frontiers of Computer Science}, 14:\penalty0 241--258, 2020.

\bibitem[Duchi and Namkoong(2018)]{duchi2018learning}
John Duchi and Hongseok Namkoong.
\newblock Learning models with uniform performance via distributionally robust
  optimization.
\newblock \emph{arXiv preprint arXiv:1810.08750}, 2018.

\bibitem[Fishburn(1975)]{fishburn1975axioms}
Peter~C Fishburn.
\newblock Axioms for lexicographic preferences.
\newblock \emph{The Review of Economic Studies}, 42\penalty0 (3):\penalty0
  415--419, 1975.

\bibitem[Ganin et~al.(2016)Ganin, Ustinova, Ajakan, Germain, Larochelle,
  Laviolette, Marchand, and Lempitsky]{ganin2016domain}
Yaroslav Ganin, Evgeniya Ustinova, Hana Ajakan, Pascal Germain, Hugo
  Larochelle, Fran{\c{c}}ois Laviolette, Mario Marchand, and Victor Lempitsky.
\newblock Domain-adversarial training of neural networks.
\newblock \emph{The journal of machine learning research}, 17\penalty0
  (1):\penalty0 2096--2030, 2016.

\bibitem[Ginosar et~al.(2015)Ginosar, Rakelly, Sachs, Yin, and
  Efros]{ginosar2015century}
Shiry Ginosar, Kate Rakelly, Sarah Sachs, Brian Yin, and Alexei~A Efros.
\newblock A century of portraits: A visual historical record of american high
  school yearbooks.
\newblock In \emph{Proceedings of the IEEE International Conference on Computer
  Vision Workshops}, pages 1--7, 2015.

\bibitem[Gupta et~al.(2020)Gupta, Yadav, and Paull]{gupta2020look}
Gunshi Gupta, Karmesh Yadav, and Liam Paull.
\newblock Look-ahead meta learning for continual learning.
\newblock \emph{Advances in Neural Information Processing Systems},
  33:\penalty0 11588--11598, 2020.

\bibitem[Hentschel et~al.(2019)Hentschel, Haas, and Tian]{hentschel2019online}
Brian Hentschel, Peter~J Haas, and Yuanyuan Tian.
\newblock Online model management via temporally biased sampling.
\newblock \emph{ACM SIGMOD Record}, 48\penalty0 (1):\penalty0 69--76, 2019.

\bibitem[Hoeffding(1994)]{hoeffding1994probability}
Wassily Hoeffding.
\newblock Probability inequalities for sums of bounded random variables.
\newblock In \emph{The collected works of Wassily Hoeffding}, pages 409--426.
  Springer, 1994.

\bibitem[Izmailov et~al.(2018)Izmailov, Podoprikhin, Garipov, Vetrov, and
  Wilson]{izmailov2018averaging}
Pavel Izmailov, Dmitrii Podoprikhin, Timur Garipov, Dmitry Vetrov, and
  Andrew~Gordon Wilson.
\newblock Averaging weights leads to wider optima and better generalization.
\newblock \emph{arXiv preprint arXiv:1803.05407}, 2018.

\bibitem[Ji et~al.(2021)Ji, Deng, Nakada, Zou, and Zhang]{ji2021power}
Wenlong Ji, Zhun Deng, Ryumei Nakada, James Zou, and Linjun Zhang.
\newblock The power of contrast for feature learning: A theoretical analysis.
\newblock \emph{arXiv preprint arXiv:2110.02473}, 2021.

\bibitem[Kay et~al.(2015)Kay, Deser, Phillips, Mai, Hannay, Strand, Arblaster,
  Bates, Danabasoglu, Edwards, Holland, Kushner, Lamarque, Lawrence, Lindsay,
  Middleton, Munoz, Neale, Oleson, Polvani, and Vertenstein]{KAYETAL15}
J.~E. Kay, C.~Deser, A.~Phillips, A.~Mai, C.~Hannay, G.~Strand, J.~M.
  Arblaster, S.~C. Bates, G.~Danabasoglu, J.~Edwards, M.~Holland, P.~Kushner,
  J.-F. Lamarque, D.~Lawrence, K.~Lindsay, A.~Middleton, E.~Munoz, R.~Neale,
  K.~Oleson, L.~Polvani, and M.~Vertenstein.
\newblock The {{Community Earth System Model}} ({{CESM}}) {{Large Ensemble
  Project}}: {{A Community Resource}} for {{Studying Climate Change}} in the
  {{Presence}} of {{Internal Climate Variability}}.
\newblock \emph{Bull. Amer. Meteor. Soc.}, pages 1333--1349, 2015.
\newblock ISSN 0003-0007, 1520-0477.
\newblock \doi{10.1175/BAMS-D-13-00255.1}.

\bibitem[Kirkpatrick et~al.(2017)Kirkpatrick, Pascanu, Rabinowitz, Veness,
  Desjardins, Rusu, Milan, Quan, Ramalho, Grabska-Barwinska,
  et~al.]{kirkpatrick2017overcoming}
James Kirkpatrick, Razvan Pascanu, Neil Rabinowitz, Joel Veness, Guillaume
  Desjardins, Andrei~A Rusu, Kieran Milan, John Quan, Tiago Ramalho, Agnieszka
  Grabska-Barwinska, et~al.
\newblock Overcoming catastrophic forgetting in neural networks.
\newblock \emph{Proceedings of the national academy of sciences}, 114\penalty0
  (13):\penalty0 3521--3526, 2017.

\bibitem[Li et~al.(2017)Li, Jamieson, DeSalvo, Rostamizadeh, and
  Talwalkar]{li2017hyperband}
Lisha Li, Kevin Jamieson, Giulia DeSalvo, Afshin Rostamizadeh, and Ameet
  Talwalkar.
\newblock Hyperband: A novel bandit-based approach to hyperparameter
  optimization.
\newblock \emph{The Journal of Machine Learning Research}, 18\penalty0
  (1):\penalty0 6765--6816, 2017.

\bibitem[Liu et~al.(2018)Liu, Simonyan, and Yang]{liu2018darts}
Hanxiao Liu, Karen Simonyan, and Yiming Yang.
\newblock Darts: Differentiable architecture search.
\newblock \emph{arXiv preprint arXiv:1806.09055}, 2018.

\bibitem[Long et~al.(2015)Long, Cao, Wang, and Jordan]{long2015learning}
Mingsheng Long, Yue Cao, Jianmin Wang, and Michael Jordan.
\newblock Learning transferable features with deep adaptation networks.
\newblock In \emph{International conference on machine learning}, pages
  97--105. PMLR, 2015.

\bibitem[Lopez-Paz and Ranzato(2017)]{lopez2017gradient}
David Lopez-Paz and Marc'Aurelio Ranzato.
\newblock Gradient episodic memory for continual learning.
\newblock \emph{Advances in neural information processing systems}, 30, 2017.

\bibitem[Malinin et~al.(2023)Malinin, andreas athanasopoulos, Barakovic,
  Cuadra, Gales, Granziera, Graziani, Kartashev, Kyriakopoulos, Lu, Molchanova,
  Nikitakis, Raina, Rosa, Sivena, Tsarsitalidis, Tsompopoulou, and
  Volf]{malinin2023shifts}
Andrey Malinin, andreas athanasopoulos, Muhamed Barakovic, Meritxell~Bach
  Cuadra, Mark Gales, Cristina Granziera, Mara Graziani, Nikolay Kartashev,
  Konstantinos Kyriakopoulos, Po-Jui Lu, Nataliia Molchanova, Antonis
  Nikitakis, Vatsal Raina, Francesco~La Rosa, Eli Sivena, Vasileios
  Tsarsitalidis, Efi Tsompopoulou, and Elena Volf.
\newblock Shifts 2.0: Extending the dataset of real distributional shifts,
  2023.
\newblock URL \url{https://openreview.net/forum?id=5RSq86IM6mE}.

\bibitem[Mallick et~al.(2022)Mallick, Hsieh, Arzani, and
  Joshi]{mallick2022matchmaker}
Ankur Mallick, Kevin Hsieh, Behnaz Arzani, and Gauri Joshi.
\newblock Matchmaker: Data drift mitigation in machine learning for large-scale
  systems.
\newblock \emph{Proceedings of Machine Learning and Systems}, 4:\penalty0
  77--94, 2022.

\bibitem[Oleson et~al.(2018)Oleson, Anderson, Jones, McGinnis, and
  Sanderson]{OLESONETAL18}
K.~W. Oleson, G.~B. Anderson, B.~Jones, S.~A. McGinnis, and B.~Sanderson.
\newblock Avoided climate impacts of urban and rural heat and cold waves over
  the {{U}}.{{S}}. using large climate model ensembles for {{RCP8}}.5 and
  {{RCP4}}.5.
\newblock \emph{Climatic Change}, pages 377--392, 2018.
\newblock ISSN 0165-0009, 1573-1480.
\newblock \doi{10.1007/s10584-015-1504-1}.

\bibitem[Oza and Russell(2001)]{oza2001online}
Nikunj~Chandrakant Oza and Stuart Russell.
\newblock \emph{Online ensemble learning}.
\newblock University of California, Berkeley, 2001.

\bibitem[Rebuffi et~al.(2017)Rebuffi, Kolesnikov, Sperl, and
  Lampert]{rebuffi2017icarl}
Sylvestre-Alvise Rebuffi, Alexander Kolesnikov, Georg Sperl, and Christoph~H
  Lampert.
\newblock icarl: Incremental classifier and representation learning.
\newblock In \emph{Proceedings of the IEEE conference on Computer Vision and
  Pattern Recognition}, pages 2001--2010, 2017.

\bibitem[Sagawa et~al.(2019)Sagawa, Koh, Hashimoto, and
  Liang]{sagawa2019distributionally}
Shiori Sagawa, Pang~Wei Koh, Tatsunori~B Hashimoto, and Percy Liang.
\newblock Distributionally robust neural networks for group shifts: On the
  importance of regularization for worst-case generalization.
\newblock \emph{arXiv preprint arXiv:1911.08731}, 2019.

\bibitem[Schwarz et~al.(2018)Schwarz, Czarnecki, Luketina, Grabska-Barwinska,
  Teh, Pascanu, and Hadsell]{schwarz2018progress}
Jonathan Schwarz, Wojciech Czarnecki, Jelena Luketina, Agnieszka
  Grabska-Barwinska, Yee~Whye Teh, Razvan Pascanu, and Raia Hadsell.
\newblock Progress \& compress: A scalable framework for continual learning.
\newblock In \emph{International conference on machine learning}, pages
  4528--4537. PMLR, 2018.

\bibitem[Shen et~al.(2022)Shen, Jones, Kumar, Xie, HaoChen, Ma, and
  Liang]{shen2022connect}
Kendrick Shen, Robbie~M Jones, Ananya Kumar, Sang~Michael Xie, Jeff~Z HaoChen,
  Tengyu Ma, and Percy Liang.
\newblock Connect, not collapse: Explaining contrastive learning for
  unsupervised domain adaptation.
\newblock In \emph{International Conference on Machine Learning}, pages
  19847--19878. PMLR, 2022.

\bibitem[Shin et~al.(2017)Shin, Lee, Kim, and Kim]{shin2017continual}
Hanul Shin, Jung~Kwon Lee, Jaehong Kim, and Jiwon Kim.
\newblock Continual learning with deep generative replay.
\newblock \emph{Advances in neural information processing systems}, 30, 2017.

\bibitem[Sinha et~al.(2017)Sinha, Namkoong, Volpi, and
  Duchi]{sinha2017certifying}
Aman Sinha, Hongseok Namkoong, Riccardo Volpi, and John Duchi.
\newblock Certifying some distributional robustness with principled adversarial
  training.
\newblock \emph{arXiv preprint arXiv:1710.10571}, 2017.

\bibitem[Sun and Saenko(2016)]{sun2016deep}
Baochen Sun and Kate Saenko.
\newblock Deep coral: Correlation alignment for deep domain adaptation.
\newblock In \emph{European conference on computer vision}, pages 443--450.
  Springer, 2016.

\bibitem[Sun et~al.(2021)Sun, Yang, Xun, and Zhang]{sun2021stagewise}
Jianhui Sun, Ying Yang, Guangxu Xun, and Aidong Zhang.
\newblock A stagewise hyperparameter scheduler to improve generalization.
\newblock In \emph{Proceedings of the 27th ACM SIGKDD Conference on Knowledge
  Discovery \& Data Mining}, pages 1530--1540, 2021.

\bibitem[Tzeng et~al.(2014)Tzeng, Hoffman, Zhang, Saenko, and
  Darrell]{tzeng2014deep}
Eric Tzeng, Judy Hoffman, Ning Zhang, Kate Saenko, and Trevor Darrell.
\newblock Deep domain confusion: Maximizing for domain invariance.
\newblock \emph{arXiv preprint arXiv:1412.3474}, 2014.

\bibitem[Webb and Zheng(2004)]{webb2004multistrategy}
Geoffrey~I Webb and Zijian Zheng.
\newblock Multistrategy ensemble learning: Reducing error by combining ensemble
  learning techniques.
\newblock \emph{IEEE Transactions on Knowledge and Data Engineering},
  16\penalty0 (8):\penalty0 980--991, 2004.

\bibitem[Wu et~al.(2021)Wu, Wang, and Huang]{wu2021frugal}
Qingyun Wu, Chi Wang, and Silu Huang.
\newblock Frugal optimization for cost-related hyperparameters.
\newblock In \emph{Proceedings of the AAAI Conference on Artificial
  Intelligence}, volume~35, pages 10347--10354, 2021.

\bibitem[Xu et~al.(2020)Xu, Zhang, Ni, Li, Wang, Tian, and
  Zhang]{xu2020adversarial}
Minghao Xu, Jian Zhang, Bingbing Ni, Teng Li, Chengjie Wang, Qi~Tian, and
  Wenjun Zhang.
\newblock Adversarial domain adaptation with domain mixup.
\newblock In \emph{Proceedings of the AAAI conference on artificial
  intelligence}, volume~34, pages 6502--6509, 2020.

\bibitem[Yao et~al.(2022{\natexlab{a}})Yao, Choi, Lee, Koh, and
  Finn]{yao2022wildtime}
Huaxiu Yao, Caroline Choi, Yoonho Lee, Pang~Wei Koh, and Chelsea Finn.
\newblock Wild-time: A benchmark of in-the-wild distribution shift over time.
\newblock In \emph{ICML 2022 Shift Happens Workshop}, 2022{\natexlab{a}}.

\bibitem[Yao et~al.(2022{\natexlab{b}})Yao, Wang, Li, Zhang, Liang, Zou, and
  Finn]{yao2022improving}
Huaxiu Yao, Yu~Wang, Sai Li, Linjun Zhang, Weixin Liang, James Zou, and Chelsea
  Finn.
\newblock Improving out-of-distribution robustness via selective augmentation.
\newblock \emph{arXiv preprint arXiv:2201.00299}, 2022{\natexlab{b}}.

\bibitem[Yue et~al.(2019)Yue, Zhang, Zhao, Sangiovanni-Vincentelli, Keutzer,
  and Gong]{yue2019domain}
Xiangyu Yue, Yang Zhang, Sicheng Zhao, Alberto Sangiovanni-Vincentelli, Kurt
  Keutzer, and Boqing Gong.
\newblock Domain randomization and pyramid consistency: Simulation-to-real
  generalization without accessing target domain data.
\newblock In \emph{Proceedings of the IEEE/CVF International Conference on
  Computer Vision}, pages 2100--2110, 2019.

\bibitem[Zenke et~al.(2017)Zenke, Poole, and Ganguli]{zenke2017continual}
Friedemann Zenke, Ben Poole, and Surya Ganguli.
\newblock Continual learning through synaptic intelligence.
\newblock In \emph{International conference on machine learning}, pages
  3987--3995. PMLR, 2017.

\bibitem[Zhang et~al.(2023)Zhang, Jia, Wang, and Wu]{zhang2023targeted}
Shaokun Zhang, Feiran Jia, Chi Wang, and Qingyun Wu.
\newblock Targeted hyperparameter optimization with lexicographic preferences
  over multiple objectives.
\newblock In \emph{The Eleventh International Conference on Learning
  Representations}, 2023.
\newblock URL \url{https://openreview.net/forum?id=0Ij9_q567Ma}.

\bibitem[Zheng et~al.(2021)Zheng, Zhao, and Oleson]{ZHENGETAL21}
Zhonghua Zheng, Lei Zhao, and Keith~W. Oleson.
\newblock Large model structural uncertainty in global projections of urban
  heat waves.
\newblock \emph{Nat Commun}, 12\penalty0 (1), 2021.
\newblock ISSN 2041-1723.
\newblock \doi{10.1038/s41467-021-24113-9}.

\end{thebibliography}
